\documentclass[11pt]{article}
\usepackage{fullpage}
\usepackage{url,natbib}
\usepackage{amsthm,amsfonts,amsmath,amssymb,epsfig,color,float,graphicx,verbatim}
\usepackage[dvipsnames]{xcolor}
\usepackage{enumitem}
\usepackage{thm-restate}
\usepackage{algorithm}
\usepackage{algpseudocode}
\usepackage{footmisc}
\usepackage{wrapfig}
\usepackage{subcaption}
\usepackage{bbm, bm}
\usepackage{natbib}

\usepackage{hyperref}
\hypersetup{
	colorlinks   = true,
	urlcolor     = blue,
	linkcolor    = blue!70!black,
	citecolor   = blue!70!black
}

\newtheorem{theorem}{Theorem}[section]
\newtheorem*{theorem*}{Theorem}

\newtheorem*{proposition*}{Proposition}
\newtheorem{example}[theorem]{Example}
\newtheorem{lemma}[theorem]{Lemma}
\newtheorem{corollary}[theorem]{Corollary}

\newtheorem{remark}[theorem]{Remark}
\newtheorem{assumption}[theorem]{Assumption}

\renewcommand{\eqref}[1]{Eq.~(\ref{#1})}

\newcommand{\E}{\mathbb{E}}

\newcommand{\NN}{\mathbb{N}}

\newcommand{\reals}{\mathbb{R}}
\newcommand{\tnab}{\widetilde{\nabla}}

\newcommand{\Lip}{\mathrm{Lip}}

\newcommand{\diam}{\mathrm{diam}}

\newcommand{\tr}{\mathrm{tr}}
\newcommand{\val}{\mathrm{val}}

\newcommand{\tbg}{\widetilde{\mathbf{g}}}
\newcommand{\tby}{\widetilde{\mathbf{y}}}

\newcommand{\Acal}{\mathcal{A}}

\newcommand{\Gcal}{\mathcal{G}}

\newcommand{\Lcal}{\mathcal{L}}
\newcommand{\Mcal}{\mathcal{M}}
\newcommand{\Ncal}{\mathcal{N}}
\newcommand{\Ocal}{\mathcal{O}}
\newcommand{\Pcal}{\mathcal{P}}
\newcommand{\Rcal}{\mathcal{R}}

\newcommand{\Scal}{\mathcal{S}}
\newcommand{\Xcal}{\mathcal{X}}
\newcommand{\Ycal}{\mathcal{Y}}

\newcommand{\bzero}{\mathbf{0}}
\newcommand{\ba}{\mathbf{a}}

\newcommand{\be}{\mathbf{e}}
\newcommand{\bg}{\mathbf{g}}
\newcommand{\bx}{\mathbf{x}}

\newcommand{\bu}{\mathbf{u}}
\newcommand{\bv}{\mathbf{v}}
\newcommand{\bw}{\mathbf{w}}
\newcommand{\by}{\mathbf{y}}
\newcommand{\bz}{\mathbf{z}}

\newcommand{\Unif}{\mathrm{Unif}}

\newcommand{\out}{\mathrm{out}}

\newcommand{\norm}[1]{\left\|#1\right\|}
\newcommand{\snorm}[1]{\|#1\|}
\newcommand{\inner}[1]{\left\langle#1\right\rangle}
\newcommand{\sinner}[1]{\langle#1\rangle}

\title{\textbf{
Differentially Private Bilevel Optimization}}
\author{Guy Kornowski
\\
Weizmann Institute of Science\thanks{Work done while interning at Apple.}
\\
\texttt{guy.kornowski@weizmann.ac.il}}

\begin{document}

\maketitle

\begin{abstract}
We present differentially private (DP) algorithms for bilevel optimization, a problem class that received significant attention lately in various machine learning applications.
These are the first algorithms for such problems under standard DP constraints, and are also the first to avoid Hessian computations which are prohibitive in large-scale settings.
Under the well-studied setting in which the upper-level is not necessarily convex and the lower-level problem is strongly-convex, our proposed gradient-based $(\epsilon,\delta)$-DP algorithm returns a point with hypergradient norm at most $\widetilde{\mathcal{O}}\left((\sqrt{d_\mathrm{up}}/\epsilon n)^{1/2}+(\sqrt{d_\mathrm{low}}/\epsilon n)^{1/3}\right)$ where $n$ is the dataset size, and $d_\mathrm{up}/d_\mathrm{low}$ are the upper/lower level dimensions.
Our analysis covers constrained and unconstrained problems alike, accounts for mini-batch gradients, and applies to both empirical and population losses.
As an application, we specialize our analysis to derive a simple private rule for tuning a regularization hyperparameter.
\end{abstract}

\section{Introduction}

Bilevel optimization (BO)
is
a fundamental framework
for
solving optimization objectives of hierarchical structure,
in which
constraints are defined themselves by an auxiliary optimization problem. Formally,
it is defined as
\begin{align} \label{eq: BO}
\tag{BO}
\mathrm{minimize}_{\bx\in\Xcal}~~&F(\bx):=f(\bx,\by^*(\bx))
    ~~~~~
    \\
    \mathrm{subject~to}~~~&\by^*(\bx)\in\arg\min_{\by}g(\bx,\by)~,
    \nonumber
\end{align}
where $F:\reals^{d_x}\to\reals$ is referred to as the {hyperobjective}, $f:\reals^{d_x}\times\reals^{d_y}\to\reals$
as the upper-level (or outer) objective,
and $g:\reals^{d_x}\times\reals^{d_y}\to\reals$ as the {lower-level} (or inner) objective. While BO is well studied for over half a century \citep{bracken1973mathematical},
it has recently received significant attention due to its diverse applications in machine learning (ML). These include hyperparameter tuning
\citep{bengio2000gradient,maclaurin2015gradient,franceschi2017forward,franceschi2018bilevel,lorraine2020optimizing,engstrom2025optimizing},
meta-learning \citep{andrychowicz2016learning,bertinetto2018meta,rajeswaran2019meta,ji2020convergence},
neural architecture search \citep{liu2018darts},
invariant learning \citep{arjovsky2019invariant,jiang2022invariant},
and data reweighting \citep{grangier2023bilevel,fan2024doge,pan2025scalebio}.
In these applications,
both the upper and lower level objectives
typically represent some loss over data,
and are given by empirical risk minimization (ERM) problems 
over a dataset $S=\{\xi_1,\dots,\xi_n\}\in\Xi^n:$\footnote{
It is possible for the datasets with respect to $f$ and $g$ to be distinct (e.g., validation/training), perhaps of different sizes.
In this case, $S$ is the entire dataset, and $n$ is the total
number of samples.
Letting $f(\cdot;\xi_i)=0$ or $g(\cdot;\xi_i)=0$ for certain indices in order to exclude corresponding data points from either objective, will not affect our results.
}
\begin{align} \label{eq: ERM}
\tag{ERM}
    f(\bx,\by)
    :=f_S(\bx,\by)=
    \frac{1}{n}\sum_{i=1}^{n}f(\bx,\by,\xi_i)
    ~,
\qquad
    g(\bx,\by)
    :=g_S(\bx,\by)
    =\frac{1}{n}\sum_{i=1}^{n}g(\bx,\by,\xi_i)
    ~,
\end{align}
where both objectives are often
proxies of stochastic (i.e., population) losses with respect to a distribution $\Pcal$
over $\Xi:$
\begin{align} \label{eq: Pop}
\tag{Pop}
    f(\bx,\by):=f_{\Pcal}(\bx,\by)=\E_{\xi\sim\Pcal}[f(\bx,\by;\xi)]~,
    \qquad
    g(\bx,\by):=g_{\Pcal}(\bx,\by)=\E_{\xi\sim\Pcal}[g(\bx,\by;\xi)]~.
\end{align}

In this work, we study bilvel optimization under
differential privacy (DP) \citep{dwork2006calibrating}.
As ML models are deployed in an ever-growing number of applications, protecting the privacy of the data on which they are trained is a major concern, and DP has become the gold-standard for privacy preserving ML \citep{abadi2016deep}.
Accordingly, DP optimization is extensively studied, with a vast literature focusing both on empirical and stochastic objectives under various settings \citep{chaudhuri2011differentially,kifer2012private,bassily2014private,wang2017differentially,bassily2019private,wang2019differentially,feldman2020private,tran2022momentum,gopi2022private,arora2023faster,carmon2023resqueing,ganesh2024private,lowy2024make}.

Nonetheless,
private algorithms for BO remain largely underexplored, and in particular, no gradient-based algorithm (aka first-order, which uses only gradient queries) that solves BO problems under DP, is known to date. This is no coincidence: until recently, no first-order methods with finite time guarantees were known even for non-private bilevel problems.
This follows the fact \citep[Lemma 2.1]{ghadimi2018approximation} that under mild regularity assumptions, the so-called hypergradient takes the form:
\begin{align} \label{eq: hypergradient 2nd order}
\nabla F(\bx)
=\nabla_x f(\bx,\by^*(\bx))-
\nabla^2_{xy}g(\bx,\by^*(\bx))[\nabla^2_{yy}g(\bx,\by^*(\bx))]^{-1}
\nabla_y f(\bx,\by^*(\bx))~.
\end{align}
Consequently, directly applying a ``gradient'' method to $F$
requires inverting Hessians of the lower-level problem at each time step, thus limiting applicability in contemporary high-dimensional applications. 
Following various approaches to tackle this challenge (see Section~\ref{sec: related}), recent breakthroughs were finally able to provide fully first-order methods for BO with non-asymptotic guarantees \citep{liu2022bome,kwon2023fully,yang2023achieving,chen2024finding}.
These recent algorithmic advancements show promising empirical results in large scale applications, even up to the LLM scale of ${\sim} 10^9$ parameters \citep{pan2025scalebio}.

On the downside,
as we will show,
first-order bilevel techniques can lead to major privacy violations.
Intuitively, this is due to possibly private information leaking between the inner and outer objectives, since $\nabla F(\bx)$ itself depends on $\by^*(\bx)$ (as seen in 
\eqref{eq: hypergradient 2nd order}).
In fact, we show in Section~\ref{sec: leak} that even a single hypergradient computation can completely break privacy, even when the upper-level objective does not depend on the data at all.

The only prior method we are aware of for DP BO was proposed by \citet{chen2024locally}, which falls short in two main aspects.
First, their algorithm only provides \emph{some} privacy guarantee which cannot be controlled by the user.
Moreover, it requires evaluating and inverting Hessians at each step, namely it is not first-order, which limits 
scalability; see Section~\ref{sec: related} for further discussion.

\subsection{Our contributions}

We present DP algorithms that solve BO problems whenever the the upper-level is smooth but not necessarily convex, and the lower-level problem is smooth and strongly-convex.
These are the
first bilevel algorithms of any sort under standard central DP constraints,
and in particular, are the first private algorithms to do so using only first-order (i.e., gradient) queries of the upper- and lower-level objectives.
Our contributions can be summarized as follows:

\begin{itemize}[leftmargin=*]
    \item \textbf{DP Bilevel ERM Algorithm (Theorem~\ref{thm: main}):}
    We present a $(\epsilon,\delta)$-DP first-order algorithm for the bilevel ERM problem (\ref{eq: BO}/\ref{eq: ERM}) that outputs with high probability a point with hypergradient norm bounded by
    \begin{equation*} \label{eq: informal ERM bound}
    \norm{\nabla F_S(\bx)}=\widetilde{\Ocal}\bigg(\Big(\frac{\sqrt{d_x}}{\epsilon n}\Big)^{1/2}+\Big(\frac{\sqrt{d_y}}{\epsilon n}\Big)^{1/3}\bigg)~.
    \end{equation*}
    Moreover, our algorithm is also applicable to the case where $\Xcal\subsetneq\reals^{d_x}$ is a non-trivial constraint set, which is common in applications of BO.\footnote{For instance, in data reweighting $\Xcal$ is the probability simplex, and
    in hyperparameter tuning it is the hyperparameter space, which is typically constrained (e.g., non-negative reals).
    }
    In the constrained setting, we obtain the same guarantee as above
    in terms of the projected hypergradient (see Section~\ref{sec: prelim} for
    details).

    \item \textbf{Mini-batch DP Bilevel ERM Algorithm (Theorem~\ref{thm: main minib}):} Aiming for a more practical algorithm, we design a variant of our previous algorithm that relies on mini-batch gradients. For the bilevel ERM problem (\ref{eq: BO}/\ref{eq: ERM}), given any batch sizes $b_{\mathrm{in}},b_{\mathrm{out}}\in\{1,\dots,n\}$
    for sampling gradients of the inner/outer problems respectively, our algorithm ensures $(\epsilon,\delta)$-DP and outputs 
    with high probability a point with hypergradient norm bounded by
    \begin{equation} \label{eq: informal ERM minib bound}
    \norm{\nabla F_S(\bx)}=\widetilde{\Ocal}\bigg(\Big(\frac{\sqrt{d_x}}{\epsilon n}\Big)^{1/2}+\Big(\frac{\sqrt{d_y}}{\epsilon n}\Big)^{1/3}+\frac{1}{b_{\mathrm{out}}}\bigg)~.
    \end{equation}
    Notably, \eqref{eq: informal ERM minib bound}
    is independent of the inner-batch size, yet depends on the outer-batch size,
    which coincides
    with known results for ``single''-level constrained nonconvex optimization \citep{ghadimi2016mini} (see Remark~\ref{rem: batch size} for further discussion).
    Our mini-batch algorithm is also applicable in the constrained setting $\Xcal\subsetneq\reals^{d_x}$ with the same guarantee in terms of projected hypergradient.
    
    \item \textbf{Population loss guarantees (Theorem~\ref{thm: population}):}
    We further provide guarantees for stochastic objectives. In particular, we show that for the population bilevel problem (\ref{eq: BO}/\ref{eq: Pop}), our $(\epsilon,\delta)$-DP algorithm outputs with high probability a point with hypergradient norm bounded by
    \[
    \norm{\nabla F_\Pcal(\bx)}=\widetilde{\Ocal}\bigg(\Big(\frac{\sqrt{d_x}}{\epsilon n}\Big)^{1/2}
    +\Big(\frac{\sqrt{d_y}}{\epsilon n}\Big)^{1/3}
    +\Big(\frac{d_x}{n}\Big)^{1/2}+\frac{1}{n^{1/4}}
    \bigg)~,
    \]
with an additional additive $1/b_{\mathrm{out}}$ factor in the mini-batch setting.

\item  \textbf{Application to private hyperparameter tuning (Section~\ref{sec: regularization tuning}):}
We specialize our algorithmic framework to the classic problem of tuning a regularization hyperparameter,
when fitting a regularized statistical model.
Our analysis of this problem leads to a simple, privacy preserving update rule for tuning the amount of regularization. As far as we know, this is the first such method which avoids selecting a hyperparameter from a predetermined set of candidates (as studied by previous works), but rather updates the hyperparameter privately on the
fly.

\end{itemize}

\subsection{Related work} \label{sec: related}

BO was introduced by \citet{bracken1973mathematical},
and grew into a vast body of work, with classical results focusing on asymptotic guarantees for certain specific problem structures
\citep{anandalingam1990solution,ishizuka1992double,white1993penalty,vicente1994descent,zhu1995optimality,ye1997exact}. There exist multiple surveys and books covering various approaches for these problems
\citep{vicente1994bilevel,dempe2002foundations,colson2007overview,bard2013practical,sinha2017review}.

\citet{ghadimi2018approximation}
observed \eqref{eq: hypergradient 2nd order} under strong-convexity of the inner problem
using the implicit function theorem, asserting that
the hypergradient can be computed via inverse Hessians,
which requires solving a linear system at each point.
Many follow up works built upon this second-order approach with additional techniques such as variance reduction, momentum, Hessian sketches, projection-free updates,
or incorporating external constraints \citep{amini2019iterative,yang2021provably,khanduri2021near,guo2021randomized,ji2021bilevel,chen2021closing,akhtar2022projection,chen2022single,tsaknakis2022implicit,hong2023two,jiang2023conditional,abolfazli2023inexact,merchav2023convex,xu2023efficient,cao2024projection,dagreou2024lower}.
Only recently, the groundbreaking result of \citet{liu2022bome} proved finite-time convergence guarantees for a fully first-order method which is based on a penalty approach. This result was soon extended to stochastic objectives \citep{kwon2023fully},
with the convergence rate later improved by \citep{yang2023achieving,chen2024finding,chen2025near},
and also extended to constrained bilevel problems \citep{yao2024constrained,kornowski2024first}.
The first-order penalty paradigm also shows promise for some bilevel problems in which the inner problem is not strongly-convex \citep{shen2023penalty,kwon2024penalty,lu2024first}, which is generally a highly challenging setting \citep{chen2024finding,bolte2025geometric}.
Moreover, \citet{pan2025scalebio} provided an efficient implementation of this paradigm, showing its effectiveness for large scale applications.

As to DP optimization, there is an extensive literature both for ERM and stochastic losses, for either convex objectives \citep{chaudhuri2011differentially,kifer2012private,bassily2014private,wang2017differentially,bassily2019private,feldman2020private,gopi2022private,carmon2023resqueing}
as well as for nonconvex objectives \citep{wang2019differentially,tran2022momentum,arora2023faster,ganesh2024private,lowy2024make,zhang2023private,kornowski2025improved}.

To the best of our knowledge, the only existing result for DP BO prior to our work is the result of \citet{chen2024locally}, which differs than ours in several aspect. Their proposed algorithm is second-order, requiring 
evaluating Hessians and inverting them at each time step, which we avoid altogether.
Moreover, \citet{chen2024locally} study the \emph{local} DP model \citep{kasiviswanathan2011can}, in which each user (i.e., $\xi_i$) does not reveal its individual information.
Due to this more challenging setting, they only derive guarantees for \emph{some finite} privacy budget $\epsilon<\infty$, even as the dataset size grows. We study the common central DP model, in which a trusted curator acts on the collected data and releases a private solution, and we are able to provide any desired privacy and accuracy guarantees with sufficiently many samples.
Our work is the first to study BO in this common DP setting.
We also note that \citet{fioretto2021differential} studied 
the related problem of DP in Stackelberg games, which are certain bilevel programs which arise in game theory, aiming at designing coordination mechanisms that maintain each agent's privacy.

After the initial version of this work became available on arXiv \citep{kornowski2024differentially},
\citet{lowy2025differentially} further studied DP BO, proposing \emph{second}-order algorithms which at the computational cost of inverting Hessians, are able to achieve gradient bounds which, interestingly, remove the dependence on the lower-level dimension $d_y$ presented in this work.

\section{Preliminaries} \label{sec: prelim}

\paragraph{Notation and terminology.}
We let bold-face letters (e.g., $\bx$) denote vectors, and denote by $\bzero$ the zero vector (whenever the dimension is clear from context) and by $I_d\in\reals^{d\times d}$ the identity matrix. 
$[n]:=\{1,2,\dots,n\}$,
$\inner{\,\cdot\,,\,\cdot\,}$ denotes the standard Euclidean dot product, and $\norm{\,\cdot\,}$ denotes either its induced norm for vectors or operator norm for matrices, and $\norm{f}_\infty=\sup_{\bx\in\Xcal}|f(\bx)|$ denotes the sup-norm.
We denote by $\mathrm{Proj}_{\mathbb{B}(\bz,R)}$ the projection onto the closed ball around $\bz$ of radius $R$.
$\Ncal(\bm{\mu},\Sigma)$ denotes a normal (i.e., Gaussian) random variable with mean $\bm{\mu}$ and covariance $\Sigma$.
We use standard big-O notation, with $\Ocal(\cdot)$ hiding absolute constants (independent of problem parameters), and $\widetilde{\Ocal}(\cdot)$ also hiding poly-logarithmic terms. We denote $f\lesssim g$ if $f=\Ocal(g)$, and $f\asymp g$ if $f\lesssim g$ and $g\lesssim f$.
A function $f:\Xcal\subseteq\reals^{d_1}\to\reals^{d_2}$ is $L_0$-Lipschitz if for all $\bx,\by\in\Xcal:\norm{f(\bx)-f(\by)}\leq L_0\norm{\bx-\by}$; $L_1$-smooth if $\nabla f$ exists and is $L_1$-Lipschitz; and $L_2$ Hessian-smooth if $\nabla^2 f$ exists and is $L_2$-Lipschitz (with respect to the operator norm).
A twice-differentiable function $f$ is $\mu$-strongly-convex if $\nabla^2 f\succeq \mu I$,
denoting by ``$\succeq$''
the standard PSD (``Loewner'') order on matrices.

\paragraph{Differential privacy.}

Two datasets $S,S'\in\Xi^n$ are said to be neighboring
if they differ by only one data point. A randomized algorithm $\Acal:\Xi^n\to\Rcal$ is called $(\epsilon,\delta)$ differentially private (or $(\epsilon,\delta)$-DP) for $\epsilon,\delta>0$ if for any two neighboring datasets $S\sim S'$ and measurable $E\subseteq\Rcal$ in the algorithm's range: $\Pr[\Acal(S)\in E]\leq e^{\epsilon}\Pr[\Acal(S')\in E]+\delta$ \citep{dwork2006calibrating}.
In Appendix~\ref{sec: dp prelim} we recall some well known DP basics used in our analyses: composition and advanced composition, the Gaussian mechanism, and privacy amplification by subsampling.

\paragraph{Gradient mapping.}
Given a point $\bx\in\reals^d$, and some $\bv\in\reals^d,~\eta>0$, we denote
$\Gcal_{\bv,\eta}(\bx):=\frac{1}{\eta}\left(\bx-\Pcal_{\bv,\eta}(\bx)\right),~\Pcal_{\bv,\eta}(\bx):=\arg\min_{\bu\in\Xcal}\big[\inner{\bv,\bu}+\frac{1}{2\eta}\norm{\bu-\bx}^2\big]$.
In particular, given an $L$-smooth function $F:\reals^d\to\reals$ and $\eta\leq\frac{1}{2L}$, we denote the projected gradient (also known as reduced gradient)
and the gradient (or proximal) mapping, respectively, as
\[
\Gcal_{ F,\eta}(\bx):=\frac{1}{\eta}\left(\bx-\Pcal_{\nabla F,\eta}(\bx)\right)~,
~~~~~
\Pcal_{\nabla F,\eta}(\bx):=\arg\min_{\bu\in\Xcal}\Big[\inner{\nabla  F(\bx),\bu}+\frac{1}{2\eta}\norm{\bu-\bx}^2\Big]
~.
\]
The projected gradient $\Gcal_{ F,\eta}(\bx)$
generalizes the gradient
to the possibly constrained setting:
for points $\bx\in\Xcal$ sufficiently far from the boundary of $\Xcal,~\Gcal_{ F,\eta}(\bx)=\nabla F(\bx)$, namely it simply reduces to the gradient; see the textbooks \citep{nesterov2013introductory,lan2020first} for additional details.

\subsection{Setting}

We impose the following assumptions which are standard in 
the BO literature.

\begin{assumption} \label{ass: main 1}
For (\ref{eq: BO}) with either (\ref{eq: ERM}) or (\ref{eq: Pop}),
we assume the following hold:
\begin{enumerate}[label=\roman*.]
    \item $\Xcal\subseteq\reals^{d_x}$ is a closed convex set;
    \item $F(\bx_0)-\inf_{\bx\in\Xcal}F(\bx)\leq \Delta_F$ for some initial point $\bx_0\in\Xcal$;
    \item $f$ is twice differentiable, and $L_1^f$-smooth;
    \item For all $\xi\in\Xi:~f(\cdot,\cdot\,;\xi)$ is $L_0^f$-Lipschitz (hence, so is $f$);
    \item $g$ is $L_2^g$-Hessian-smooth, and for all $\bx\in\Xcal:~g(\bx,\cdot)$ is $\mu_g$-strongly-convex;
    \item  For all $\xi\in\Xi:~g(\cdot,\cdot\,;\xi)$ is $L_1^g$-smooth (hence, so is $g$).
\end{enumerate}
\end{assumption}

As mentioned, these assumptions are standard in the study of BO problems and are shared by nearly all of the previously discussed works. In particular, the strong convexity of $g(\bx,\cdot)$ ensures that $\by^*(\bx)$ is uniquely defined, which is generally required in establishing the regularity of the hyperobjective. Indeed, it is known that dropping this assumption, can, in general, lead to pathological behaviors not amenable for algorithmic guarantees
(cf. \citealt{chen2024finding,bolte2025geometric} and discussions therein).
For the purpose of differential privacy though, the strong convexity of $g(\bx,\cdot)$ raises a subtle issue. As the standard assumption in the DP optimization literature is that the component functions are Lipschitz, which allows privatizing gradients using sensitivity arguments, strongly-convex objectives cannot be Lipschitz over the entire Euclidean space.\footnote{If $g(\bx,\cdot\,;\xi)$ were Lipschitz over $\reals^{d_y}$ for all $\xi\in\Xi$, then so would $g(\bx,\cdot)$, contradicting strong convexity.} Therefore, strongly-convex objectives are regularly analyzed in the DP setting under the additional assumption that the domain of interest is bounded. For bilevel problems, the domain of interest for $\by$ is the lower-level solution set, thus we impose the following assumption.

\begin{assumption} \label{ass: main 2}
    There exists a compact set $\Ycal\subset\reals^{d_y}$ with $\{\by^*(\bx):\bx\in\Xcal\}\subseteq \Ycal$, such that for all $\bx\in\Xcal,\,\xi\in\Xi:~g(\bx,\cdot\,;\xi)$ is $L_0^g$-Lipschitz over $\Ycal$.
\end{assumption}

\begin{remark} \label{remark: diamY}
    Note that $\diam(\Ycal)\leq {2 L_0^g}/{\mu_g}$: Indeed, fixing some $\bx\in\Xcal$, $g(\bx,\cdot\,;\xi)$ is $L_0^g$-Lipschitz over $\Ycal$ for all $\xi\in\Xi$, thus so is $g(\bx,\cdot)$. By $\mu_g$-strong-convexity, we get for all $\by\in\Ycal:~\mu_g\norm{\by-\by^*(\bx)}$ $\leq\norm{\nabla_y g(\bx,\by)}\leq L_0^g$. Hence $\Ycal\subseteq\mathbb{B}(\by^*(\bx),{L_0^g}/{\mu_g})$, which is of diameter ${2L_0^g}/{\mu_g}$.
\end{remark}

Following Assumptions \ref{ass: main 1} and \ref{ass: main 2}, we denote $\ell:=\max\{L_0^f,L_1^f,L_0^g,L_1^g,L_2^g\},~\kappa:=\ell/\mu_g$.

\section{Warm up: Privacy can leak from lower to upper level} \label{sec: leak}

In this section,
we provide an illuminating example of privacy breaking in first-order BO,
which
motivates the algorithmic approach presented
later in the paper.
In particular, we exemplify that even when the upper-level problem supposedly does not depend on the dataset at all, the bilevel structure can give rise to \emph{hyper}-gradients that are non-private.

\begin{example}
Given any dataset $S=\{\xi_1,\dots,\xi_n\}\subset\reals^d$,
consider the BO problem corresponding to the objectives over $\reals^{2d}:$
\begin{align*}
    f(\bx,\by)=\frac{1}{2}\norm{\bx+\by}^2
    ~,
    \qquad
    g(\bx,\by)=\frac{1}{2}\sum_{i=1}^{n}\norm{\by-\xi_i}^2~.
\end{align*}
These are smooth non-negative objectives, and $g$ is strongly-convex with respect to $\by$, and so we see that Assumptions~\ref{ass: main 1} and \ref{ass: main 2} hold.
Further note that the upper-level objective $f$ does not depend on the data $S$ at all.
Nonetheless, we will show that a single hypergradient computation can break privacy.

To see this, note that for any $\bx\in\reals^d:~\by^*(\bx)=\arg\min_{\by}g(\bx,\by)=\frac{1}{n}\sum_{i=1}^{n}\xi_i$
since the squared loss is minimized by the average.
Furthermore, $\nabla_x f(\bx,\by)=\bx+\by$
and $\nabla^2_{xy} g\equiv\bzero$, so by \eqref{eq: hypergradient 2nd order}
we get that
$\nabla F(\bx)=\bx+\by^*(\bx)
=\bx+\frac{1}{n}\sum_{i=1}^{n}\xi_i$.
In particular, computing the hypergradient at the origin simply returns the dataset's mean
$\nabla F(\bzero)=\frac{1}{n}\sum_{i=1}^{n}\xi_i$,
which is well-known to be a non-private statistic of the dataset $S$ (cf. \citealp{kamath2019privately}).
\end{example}

The algorithms to follow incorporate a fix to overcome such privacy leaks: instead of ever computing $\by^*(\bx)$ (as non-private algorithms do up to exponentially small precision), we estimate it using an auxiliary private method, and use the resulting private point
$\tby^*(\bx)$
to estimate the hypergradient at $\bx$.\footnote{To be precise, in order to also avoid Hessian computations, $\nabla F(\bx)$ will be estimated by the difference of gradients at two private points close to $\by^*(\bx)$.}
This leads to non-negligible bias in the hypergradient estimators compared to non-private first-order BO,
that we need to account for in our analysis.
We further discuss this in Section~\ref{sec: sketch}.

\section{Algorithm for DP bilevel ERM} \label{sec: ERM fullb}

In this section, we consider the bilevel ERM problem (\ref{eq: BO}/\ref{eq: ERM}), for which we denote the hyperobjective by $F_S$.
Our algorithm is presented in Algorithm~\ref{alg: main}. We prove the following result:

\begin{theorem} \label{thm: main}
Assume \ref{ass: main 1} and \ref{ass: main 2} hold, and that
$\alpha\leq \ell\kappa^{3}\min\{\frac{1}{2\kappa},\frac{L_0^g}{L_0^f},\frac{L_1^g}{L_1^f},\frac{\Delta_F}{\ell\kappa}\}$.
Then there is a  parameter assignment
$\lambda\asymp\ell\kappa^3\alpha^{-1},~\sigma^2\asymp\ell^2\kappa^2T\log(T/\delta)\epsilon^{-2}n^{-2},~\eta\asymp\ell^{-1}\kappa^{-3},~T\asymp\Delta_F \ell\kappa^{3}\alpha^{-2}$,
such that running Algorithm~\ref{alg: main} satisfies $(\epsilon,\delta)$-DP,
and returns $\bx_{\out}$ such that with probability at least $1-\gamma:$
\[
\norm{\Gcal_{ F_S,\eta}(\bx_{\out})}
\leq\alpha=\widetilde{\Ocal}\bigg(
K_1\Big(\frac{\sqrt{d_x}}{\epsilon n}\Big)^{1/2}
+K_2\Big(\frac{\sqrt{d_y}}{\epsilon n}\Big)^{1/3}
\bigg)
~,
\]
where $K_1=\Ocal(\Delta_F^{1/4}\ell^{3/4}\kappa^{5/4}),~K_2=\Ocal(\Delta_F^{1/6}\ell^{5/6}\kappa^{11/6})$.
\end{theorem}

\begin{remark}
Recall that when $\Xcal=\reals^{d_x}$, then $\Gcal_{F_S,\eta}(\bx_{\out})=\nabla F_S(\bx_{\out})$.
\end{remark}

\begin{remark}
Algorithm~\ref{alg: main} generalizes previously studied algorithms in two special cases. First, in the non-private case (when $\epsilon=\infty$ or $\delta=1$) the algorithm can skip the private inner loop and simply solve the inner problems instead in $\widetilde{\Ocal}(1)$ steps, thus reducing to the first-order non-private bilevel algorithms analyzed by \citep{kwon2023fully,chen2024finding,chen2025near}.
Second, if the lower-level objective is independent of $\bx$, the partial derivative $\nabla_x g\equiv \bzero$ vanishes and ${\tby}_t
\approx
\arg\min_{\by}g(\bx_t,\by)$ remains constant independently of $t$. In this case, Algorithm~\ref{alg: main} simply reduces to DP-GD with respect to the first variable of $f$, and indeed, the $\Ocal(({\sqrt{d_x}}/{\epsilon n})^{1/2})$ gradient bound we obtain matches the known result of DP-GD for smooth-nonconvex ``single-level'' optimization \citep{wang2017differentially}.
\end{remark}

\begin{algorithm}[ht]
\begin{algorithmic}[1]\caption{DP Bilevel
}\label{alg: main}
\State \textbf{Input:}
Initialization $(\bx_0,\by_0)\in\Xcal\times\Ycal$,
privacy budget $(\epsilon,\delta)$,
penalty $\lambda>0$, noise level $\sigma^2>0$, step size $\eta>0$, iteration budget $T\in\NN$.
\For{$t=0,\dots,T-1$}
\State
Apply $(\frac{\epsilon}{\sqrt{18T}},\frac{\delta}{3(T+1)})$-DP-Loc-GD 
(Algorithm~\ref{alg: DPLocGD})
to solve
\Comment{Strongly-convex problems}
\begin{align*}
{\tby}_t
\approx
\arg\min_{\by}g(\bx_t,\by)~,
~~~~~
\tby^\lambda_t\approx
\arg\min_{\by}\left[f(\bx_t,\by)+\lambda\cdot g(\bx_t,\by)\right]
\end{align*}
\State $\tbg_t=\nabla_x f(\bx_t,\tby^\lambda_t)+\lambda\left(\nabla_x g(\bx_t,\tby^\lambda_t)-\nabla_x g(\bx_t,\tby_t)\right)+\nu_t$, where $\nu_t\sim\Ncal(\bzero,\sigma^2 I_{d_x})$
\State $\bx_{t+1}=\arg\min_{\bu\in\Xcal}
\left\{\inner{\tbg_t,\bu}+\frac{1}{2\eta}\norm{\bu-\bx_{t}}^2\right\}$
\Comment{If $\Xcal=\reals^{d_x}$, then $\bx_{t+1}=\bx_{t}-\eta\tbg_t$}
\EndFor
\State $t_{\out}=\arg\min_{t\in\{0,\dots,T-1\}}\norm{\bx_{t+1}-\bx_t}$.
\State \textbf{Output:} $\bx_{t_\out}$. 
\end{algorithmic}
\end{algorithm}

\begin{algorithm}[ht]
\begin{algorithmic}[1]\caption{DP-Loc-GD
}\label{alg: DPLocGD}
\State \textbf{Input:} Objective $h:\reals^{d_y}\to\reals$,
initialization $\by_0\in\Ycal$,
privacy budget $(\epsilon',\delta')$,
number of rounds $M\in\NN$, noise level $\sigma_\mathrm{GD}^2>0$, step sizes $(\eta_{t})_{t=0}^{T-1}$, iteration budget $T_{\mathrm{GD}}\in\NN$, radii $(R_m)_{m=0}^{M-1}>0$.
\State $\by^0_{0}=\by_0$
\For{$m=0,\dots,M-1$}
\For{$t=0,\dots,T_\mathrm{GD}-1$}
\State $\by_{t+1}^m=\mathrm{Proj}_{\mathbb{B}(\by_0^m,R_m)}\left[\by_t^m-\eta_t\left(\nabla h(\by_t^m)+\nu_t\right)\right]$, where $\nu_t^m\sim\Ncal(\bzero,\sigma_{\mathrm{GD}
}^2 I_{d_y})$
\EndFor
\State $\by_0^{m+1}=\frac{1}{T}\sum_{t=0}^{T-1}\by_t^m$
\EndFor
\State \textbf{Output:} $\by_{\out}=\by^M_0$. 
\end{algorithmic}
\end{algorithm}

\subsection{Analysis overview} \label{sec: sketch}

In this section, we will go over the main ideas that appear in the proof of Theorem~\ref{thm: main}, all of which are provided in full detail in Section~\ref{sec: proofs}. We start by introducing some useful notation:
Given $\lambda>0$, we denote the penalty function
\begin{gather*}
\Lcal_\lambda(\bx,\by):=f(\bx,\by)+\lambda\left[g(\bx,\by)-g(\bx,\by^*(\bx))\right]~,
    \\
\text{and}~~~\Lcal^*_\lambda(\bx):=\Lcal_\lambda(\bx,\by^\lambda(\bx))~,
~~~~\by^\lambda(\bx):=\arg\min_{\by}\Lcal_\lambda(\bx,\by)~.
\end{gather*}
The starting point of our analysis is the following result:

\begin{lemma}[\citealt{kwon2023fully,chen2024finding,chen2025near}]
\label{lem: kwon}
If
$\lambda\geq 2L_1^f/\mu_g$,
then:~~(i) $\norm{\Lcal^*_\lambda- F}_{\infty}
    =\Ocal(\ell\kappa/\lambda)$;~~(ii) $\norm{\nabla \Lcal^*_\lambda-\nabla F}_{\infty}=\Ocal(\ell\kappa^3/\lambda)
    $;~~(iii) $\Lcal_\lambda^*$ is $\Ocal(\ell\kappa^3)$-smooth 
    (independently of $\lambda$).
\end{lemma}

In other words, the lemma shows that 
for sufficiently large penalty $\lambda$,
$\Lcal_\lambda^*$ is a smooth approximation of the hyperobjective $F$, and that it suffices to minimize the gradient norm of $\Lcal_\lambda^*$ in order to get a hypergradient guarantee in terms of $\nabla F$. Moreover,
had we computed $\by^*(\bx),\by^\lambda(\bx)$,
we note that $\nabla\Lcal_\lambda^*$ can be obtained in a first-order fashion, since by construction $\Lcal_\lambda^*(\bx)=\arg\min_{\by}\Lcal_\lambda(\bx,\by)$, and therefore by first-order optimality:
\begin{align}
\nabla\Lcal_\lambda^*(\bx)
&=\nabla_x\Lcal^*_\lambda(\bx,\by^\lambda(\bx))+\nabla_x \by^\lambda(\bx)^{\top}
\underset{=\bzero}{\underbrace{\nabla_y\Lcal_\lambda(\bx,\by^\lambda(\bx))}}
\nonumber
\\
&=\nabla_x f(\bx,\by^\lambda(\bx))+\lambda\left(\nabla_x g(\bx,\by^\lambda(\bx))-\nabla_{x}g(\bx,\by^*(\bx))\right)~. \label{eq: first order}
\end{align}
This strategy raises a privacy concern, as we exemplified in Section~\ref{sec: leak}: Since $\by^*(\bx),\by^\lambda(\bx)$ are required in order to compute $\nabla\Lcal_\lambda^*(\bx)$, and are defined as the minimizers of $g(\bx,\cdot),\Lcal_\lambda(\bx,\cdot)$ which are data-dependent, we cannot simply compute them
under DP. 
In other words, even deciding \emph{where} to invoke the gradient oracles, can already 
leak user information, hence
breaking privacy
before the gradients are even computed.
We therefore must resort to approximating them using an auxiliary private method, for which we use DP-Loc(alized)-GD (Algorithm \ref{alg: DPLocGD}).\footnote{
We can
replace the inner solver by
any DP method that guarantees with high probability the optimal rate for strongly-convex objectives, as we further discuss in Appendix~\ref{sec: DPGD}.
}

In our analysis, we crucially rely on the fact that $g(\bx,\cdot)$ $\Lcal_\lambda(\bx,\cdot)$ are both strongly-convex, which implies that optimizing them produces ${\tby}_t,\tby^\lambda_t$ such that the \emph{distances} to the minimizers, namely $\|{\tby}_t-\by^*(\bx_t)\|,\|{\tby^\lambda_t-\by^\lambda(\bx_t)}\|$, are small.
The distance bound is key, since then \eqref{eq: first order} allows using the smoothness of $f,g$ to
translate the distance bounds into
an inexact (i.e., biased) gradient oracle for $\nabla\Lcal_\lambda^*(\bx_t)$,
computed at the private points ${\tby}_t,\tby^\lambda_t$.
We are able to bound the gradient bias induced by privatizing the inner solvers, as follows:

\begin{restatable}{lemma}{propInexact}\label{lem: inexact grad}
If $\lambda\geq\max\{\frac{2L_1^g}{\mu_g},\frac{L_0^f}{L_0^g},\frac{L^f_1}{L^g_1}\}$, then the random variables $\tby_t,\tby^\lambda_t$ as defined in Algorithm~\ref{alg: main} satisfy for all $t<T$
with probability at least $1-\gamma:$
\[
\norm{\nabla\Lcal^*_\lambda(\bx_t)-\left[\nabla_x f(\bx_t,\tby^\lambda_t)+\lambda\left(\nabla_x g(\bx_t,\tby^\lambda_t)-\nabla_x g(\bx_t,\tby_t)\right)\right]}\leq \beta
=\widetilde{\Ocal}\left(\frac{\lambda\ell\kappa\sqrt{d_y T}}{\epsilon n}\right)
~.
\]
\end{restatable}

Having constructed an inexact hypergradient oracle, we can privatize its response using the standard Gaussian mechanism.
To do so, we recall that the noise variance required to ensure privacy is tied to the component Lipschitz constants,
and note that  $\Lcal_\lambda^*(\bx)$ decomposes as the finite-sum
\[
\Lcal_\lambda^*(\bx)=\frac{1}{n}\sum_{i=1}^{n}\Lcal^*_{\lambda,i}(\bx)~,
~~~~
\Lcal^*_{\lambda,i}(\bx):= f(\bx,\by^\lambda(\bx);\xi_i)+\lambda\left[g(\bx,\by^\lambda(\bx);\xi_i)-g(\bx,\by^*(\bx);\xi_i)\right]~.
\]
Bounding the Lipschitz constant of $\Lcal_{\lambda,i}^*$ naively by applying the chain rule would result in the bound
$\Lip(\by^*)(L_0^f+2\lambda L_0^g)\lesssim\lambda  \Lip(\by^*)L_0^g$, where $\Lip(\by^*)$ is the Lipschitz constant of $\by^*(\bx):\reals^{d_x}\to\reals^{d_y}$. Unfortunately, this bound grows with the penalty parameter $\lambda$, which will eventually be set large, and in particular, will grow with the dataset size $n$. We therefore need to derive a more nuanced analysis that allows obtaining a significantly tighter Lipschitz bound. We achieve this by showing that $\|\by^\lambda(\bx)-\by^*(\bx)\|\lesssim 1/\lambda$, which cancels out the multiplication by $\lambda$,
and thus prove the following tight Lipschitz bound:

\begin{restatable}{lemma}{LiImprovedLip}
\label{lem: Li improved Lip}
$\Lcal^*_{\lambda,i}$ is 
$\Ocal(\ell\kappa)$-Lipschitz independently of $\lambda$.
\end{restatable}

Finally, having constructed a private inexact stochastic oracle response for the smooth approximation $\Lcal^*_\lambda$, we analyze an outer loop (update of $\bx_t$), showing that is optimally robust to inexact and noisy gradients.
Moreover, we employ an output rule which makes use of the already-privatized iterates, thus overcoming the need of additional noise in choosing the smallest gradient.
This is a subtle technical challenge which is unique to the private setting, since non-private algorithms can simply validate the gradient norm and return the minimal one, yet in our case, using the Laplace mechanism to do so would spoil the convergence rate.
In particular, we show that the corresponding trajectory and output rule lead a point with small (projected-)gradient norm, as stated below:

\begin{restatable}{proposition}{propOuterAlg}\label{prop: outer alg}
Suppose $h:\reals^d\to\reals$ is $L$-smooth, that $\snorm{\tnab h(\cdot)-\nabla h(\cdot)}\leq\beta$, and consider the following update rule with $\eta=\frac{1}{2L}:
\bx_{t+1}=\arg\min_{\bu\in\Xcal}
\big\{\big\langle\tnab h(\bx_{t})+\nu_t,\bu\big\rangle+\frac{1}{2\eta}\norm{\bx_{t}-\bu}^2\big\}~,
~\nu_{t}\sim\Ncal(0,\sigma^2 I)
$
with the output rule $\bx_\out:=\bx_{t_{\out}},~t_{\out}:={\arg\min}_{t\in\{0,\dots,T-1\}}\norm{\bx_{t+1}-\bx_t}$.
If $\alpha>0$ is such that
$\alpha\geq C\max\{\beta,\sigma\sqrt{d\log(T/\gamma)}\}$ for a sufficiently large absolute constant $C>0$, then with probability at least $1-\gamma:~\norm{\Gcal_{h,\eta}(\bx_{\out})}\leq\alpha$
for $T=\Ocal\big(\frac{L(h(\bx_0)-\inf h)}{\alpha^2}\big)$.  
\end{restatable}

Notably, 
for any $\alpha>0$,
the result above allows the bias $\beta$ to be as large as $\Omega(\alpha)$ at all points, and the algorithm still reaches an $\alpha$-stationary point.
Overall, applying Proposition~\ref{prop: outer alg} to $h=\Lcal_\lambda^*$, we see that the (projected-)gradient norm can be as small as $\max\{\beta,\sigma\sqrt{d_x}\}$, up to logarithmic terms. Accounting for the smallest possible 
inexactness $\beta$ and noise addition $\sigma$ that ensure
the the inner and outer loops, respectively, are both sufficiently private, we conclude the proof of Theorem~\ref{thm: main};
the full details appear in Section~\ref{sec: proofs}.

\section{Mini-batch algorithm for DP bilevel ERM}
\label{sec: minib}

In this section, we consider once again the bilevel ERM problem (\ref{eq: BO}/\ref{eq: ERM}),
and provide Algorithm~\ref{alg: minibatch},
which is a mini-batch variant of the previously discussed bilevel ERM algorithm.
Given a mini-batch $B\subseteq S=\{\xi_1,\dots,\xi_n\}$ and a function $h:\reals^d\times\Xi\to\reals$, we let
$\nabla h(\bz;B)=\frac{1}{|B|}\sum_{\xi_i\in B}\nabla h(\bz;\xi_i)$ denote the mini-batch gradient.
We prove the following result:

\begin{theorem} \label{thm: main minib}
Assume \ref{ass: main 1} and \ref{ass: main 2} hold, and that $\alpha\leq \ell\kappa^{3}\min\{\frac{1}{2\kappa},\frac{L_0^g}{L_0^f},\frac{L_1^g}{L_1^f},\frac{\Delta_F}{\ell\kappa}\}$.
Then running Algorithm~\ref{alg: minibatch} with parameters assigned as in Theorem~\ref{thm: main}
and any batch sizes $b_{\mathrm{in}},b_{\mathrm{out}}\in[n]$
for sampling gradients,
satisfies $(\epsilon,\delta)$-DP
and returns $\bx_{\out}$ such that with probability at least $1-\gamma:$
\[
\norm{\Gcal_{ F_S,\eta}(\bx_{\out})}
\leq \alpha
=\widetilde{\Ocal}\bigg(
K_1\Big(\frac{\sqrt{d_x}}{\epsilon n}\Big)^{1/2}
+K_2\Big(\frac{\sqrt{d_y}}{\epsilon n}\Big)^{1/3}
+K_3\cdot\frac{1}{b_{\mathrm{out}}}\bigg)
~,
\]
where 
$K_1=\Ocal(\Delta_F^{1/4}\ell^{3/4}\kappa^{5/4}),~K_2=\Ocal(\Delta_F^{1/6}\ell^{5/6}\kappa^{11/6}),~K_3=\Ocal(\ell\kappa)$.
\end{theorem}

\begin{remark}[Outer batch size dependence] \label{rem: batch size}
Algorithm~\ref{alg: minibatch} ensures privacy for any batch sizes, yet notably, the guaranteed gradient norm bound does not go to zero (as $n\to\infty$) for constant outer-batch size.
The same phenomenon also holds for for ``single''-level constrained nonconvex optimization, as noted by \citet{ghadimi2016mini} (specifically, see Corollary 4 and related discussion therein).
Accordingly, the inner-batch size
$b_{\mathrm{in}}$ can be set whatsoever,
while
setting $b_{\mathrm{out}}=\Ocal(\max\{(\epsilon n/\sqrt{d_x})^{1/2},(\epsilon n/\sqrt{d_y})^{1/3}\})\ll n$ recovers the full-batch rate (in general, $\|\Gcal(\bx_\out)\|\overset{n\to\infty}{\longrightarrow}0$
whenever
$b_{\mathrm{out}}\overset{n\to\infty}{\longrightarrow}\infty$).
It is interesting to note that the $1/b_{\mathrm{out}}$ term shows up in the analysis only as an upper bound on the sub-Gaussian norm of the mini-batch gradient estimator.
Thus, in applications for which some (possibly small) batch size results in reasonably accurate gradients, the result above holds with the outer mini-batch gradient's standard deviation replacing $1/b_{\mathrm{out}}$, which is to be expected anyhow for high probability guarantees.
\end{remark}

\begin{algorithm}[h]
\begin{algorithmic}[1]\caption{Mini-batch DP Bilevel
}\label{alg: minibatch}
\State \textbf{Input:}
Initialization $(\bx_0,\by_0)\in\Xcal\times\Ycal$,
privacy budget $(\epsilon,\delta)$,
penalty $\lambda>0$, noise level $\sigma^2>0$, step size $\eta>0$, iteration budget $T\in\NN$, batch sizes $b_{\mathrm{in}},b_{\mathrm{out}}\in\NN$.
\For{$t=0,\dots,T-1$}
\State Apply $(\frac{\epsilon}{\sqrt{18T}},\frac{\delta}{3(T+1)})$-DP-Loc-SGD 
(Algorithm~\ref{alg: DPLocSGD})
to solve
\Comment{Strongly-convex problems}
\begin{align*}
{\tby}_t
\approx
\arg\min_{\by}g(\bx_t,\by)~,
~~~~~
\tby^\lambda_t\approx
\arg\min_{\by}\left[f(\bx_t,\by)+\lambda\cdot g(\bx_t,\by)\right]
\end{align*}
\State $\tbg_t=
\nabla_x f(\bx_t,\tby^\lambda_t;B_t)+\lambda\left(\nabla_x g(\bx_t,\tby^\lambda_t;B_t)-\nabla_x g(\bx_t,\tby_t;B_t)\right)+\nu_t,~B_t\sim S^{b_{\mathrm{out}}},~\nu_t\sim\Ncal(\bzero,\sigma^2 I_{d_x})$
\State $\bx_{t+1}=\arg\min_{\bu\in\Xcal}
\left\{\inner{\tbg_t,\bu}+\frac{1}{2\eta}\norm{\bu-\bx_{t}}^2\right\}$
\Comment{If $\Xcal=\reals^{d_x}$, then $\bx_{t+1}=\bx_{t}-\eta\tbg_t$}
\EndFor
\State $t_{\out}=\arg\min_{t\in\{0,\dots,T-1\}}\norm{\bx_{t+1}-\bx_t}$. 
\State \textbf{Output:} $\bx_{t_\out}$. 
\end{algorithmic}
\end{algorithm}

\begin{algorithm}[h]
\begin{algorithmic}[1]\caption{DP-Loc-SGD
}\label{alg: DPLocSGD}
\State \textbf{Input:} Objective $h:\reals^{d_y}\times\Xi\to\reals$,
initialization $\by_0\in\Ycal$,
privacy budget $(\epsilon',\delta')$,
batch size $b_{\mathrm{in}}\in\NN$,
number of rounds $M\in\NN$, noise level $\sigma_\mathrm{SGD}^2>0$, step sizes $(\eta_{t})_{t=0}^{T-1}$, iteration budget $T_{\mathrm{SGD}}\in\NN$, radii $(R_m)_{m=0}^{M-1}>0$.
\State $\by^0_{0}=\by_0$
\For{$m=0,\dots,M-1$}
\For{$t=0,\dots,T_\mathrm{SGD}-1$}
\State $\by_{t+1}^m=\mathrm{Proj}_{\mathbb{B}(\by_0^m,R_m)}\left[\by_t^m-\eta_t\left(\nabla h(\by_t^m;B_t)+\nu_t\right)\right],~~B_t^m\sim S^{b_\mathrm{in}},~\nu_t^m\sim\Ncal(\bzero,\sigma_{\mathrm{SGD}
}^2 I_{d_y})$
\EndFor
\State $\by_0^{m+1}=\frac{1}{T}\sum_{t=0}^{T-1}\by_t^m$
\EndFor
\State \textbf{Output:} $\by_{\out}=\by^M_0$. 
\end{algorithmic}
\end{algorithm}

The difference between Algorithm~\ref{alg: minibatch} and Algorithm~\ref{alg: main},
is that both the inner and outer loops sample mini-batch gradients.
The inner loop guarantee is the same regardless of the inner-batch size $b_{\mathrm{in}}$, since for strongly-convex objectives it is possible to prove the same
convergence rate
guarantee for DP optimization in any case (as further discussed in Appendix~\ref{sec: DPGD}).
As to the outer loop ($\bx_{t+1}$ update rule), 
we apply standard concentration bounds to argue about the quality of the gradient estimates --- hence the additive $1/b_{\mathrm{out}}$ factor ---
and rely on our analysis of the outer loop
with inexact gradients (which are now even less exact due to sampling stochasticity).
We remark that compared to the well-known analysis of \citet{ghadimi2016mini} for mini-batch constrained nonconvex optimization, we derive high probability bounds without requiring several re-runs of the algorithm.
We further remark that mini-batch sampling with replacement is analyzed for simplicity, though the same guarantees (up to constants) can be derived for sampling without replacement, at the cost of a more involved analysis.

\section{Generalizing from ERM to population loss}

We now move to consider stochastic (population) objectives, namely the problem (\ref{eq: BO}/\ref{eq: Pop}).
We denote the population hyperobjective by $F_\Pcal$, and as before $F_S$ denotes the empirical objective, where $S\sim\Pcal^n$. We prove the following result:

\begin{theorem} \label{thm: population}
Under Assumptions \ref{ass: main 1} and \ref{ass: main 2},
if the preconditions of Theorem~\ref{thm: main} hold,
then Algorithm~\ref{alg: main} is $(\epsilon,\delta)$-DP,
and returns $\bx_{\out}$ such that with probability at least $1-\gamma:$
\[
\norm{\Gcal_{ F_{\Pcal},\eta}(\bx_{\out})}
\leq\alpha=\widetilde{\Ocal}\bigg(
K_1\Big(\frac{\sqrt{d_x}}{\epsilon n}\Big)^{1/2}
+K_2\Big(\frac{\sqrt{d_y}}{\epsilon n}\Big)^{1/3}
+\mathrm{gen}_n
\bigg)
~,
\]
where 
$K_1=\Ocal(\Delta_F^{1/4}\ell^{1/4}\kappa^{5/4}),~K_2=\Ocal(\Delta_F^{1/6}\ell^{5/6}\kappa^{11/6})$, and
\[
\mathrm{gen}_n=\widetilde{\Ocal}\bigg(\frac{\ell\kappa(\ell\kappa+1)(\sqrt{d_x}+\kappa)}{\sqrt{n}}+\frac{\ell^{1/2}\kappa^{3/2}(\ell\kappa+1)}{n^{1/4}}
\bigg)~.
\]
Similarly, if the preconditions of Theorem~\ref{thm: main minib} hold, then for any batch sizes $b_{\mathrm{in}},b_{\mathrm{out}}\in[n]$,
Algorithm~\ref{alg: minibatch} is $(\epsilon,\delta)$-DP, and returns $\bx_{\out}$ such that with probability at least $1-\gamma:$
\[
\norm{\Gcal_{ F_{\Pcal},\eta}(\bx_{\out})}
\leq\alpha=\widetilde{\Ocal}\bigg(
K_1\Big(\frac{\sqrt{d_x}}{\epsilon n}\Big)^{1/2}
+K_2\Big(\frac{\sqrt{d_y}}{\epsilon n}\Big)^{1/3}
+\frac{\ell\kappa}{b_{\mathrm{out}}}
+\mathrm{gen}_n \bigg)
~.
\]
\end{theorem}

The proof of Theorem~\ref{thm: population} relies on establishing uniform convergence of the empirical hypergradients to the population hypergradients (Lemma~\ref{lem: uc}), which as far as we know does not appear in prior literature and therefore may be of independent interest.
This further implies uniform convergence of projected gradients, since the gradient mapping is non-expansive.

\section{Application: private regularization hyperparameter tuning}
\label{sec: regularization tuning}

In this section, we specialize our private bilevel framework to the well-studied problem of tuning a regularization hyperparameter.
Specifying our analysis to this task
results in a simple private algorithm, which as we will soon explain, offers an alternative paradigm for private hyperparameter tuning compared to existing works on this topic \citep{liu2019private,papernot2022hyperparameter}.

We start by introducing the setting.
Given a labeled dataset $S=(\ba_i,b_i)_{i=1}^{n}\subset \reals^{d_a}\times\reals$, we aim to fit a parametric model $\{h_\theta:\reals^{d_a}\to\reals\mid\theta\in\Theta\subseteq\reals^{d_\theta}\}$ (e.g., linear model, or neural network)
with respect to a loss function $\ell oss:\reals\times\reals\to\reals_{\geq 0}$.
A classic approach is to split the data into two disjoint datasets, train and validation
$S=S_{\tr}\cup S_{\val}$, choose a regularizer $ R:\Theta\to\reals_{\geq 0}$ (e.g., $R(\theta)=\|\theta\|_2$ or $R(\theta)=\|\theta\|_1$) and minimize the validation loss of a model trained under regularization, while searching the best amount of regularization. This corresponds to
the bilevel problem

\begin{align}
    \mathrm{minimize}_{\omega\geq 0} ~~~~F(\omega)&:=
    f(\omega,\theta^*_\omega)=
    \frac{1}{|S_{\val}|}
\sum_{(\ba_i,b_i)\in S_{\val}}\ell oss(h_{\theta^*_\omega}(\ba_i),b_i)
\nonumber    \\
\mathrm{subject~to}~~~~~~~~~\theta^*_{\omega}&:=\arg\min_{\theta\in\Theta}g(\omega,\theta)=
\frac{1}{|S_{\tr}|}\sum_{(\ba_i,b_i)\in S_{\tr}}\ell oss(h_\theta(\ba_i),b_i)+\omega\cdot R(\theta)~.
\label{eq: hyperparameter bilevel}
\end{align}

The so-called model selection problem described above is extremely well studied, and strategies to solve it generally fall into two categories.
The first approach is to discretize the domain of $\omega$ to a finite set $\Omega$, train models associated to each $\omega\in\Omega$, and pick the model minimizing the validation loss (cf. \citealp[Section 7]{hastie2009elements} for a detailed discussion). 
This raises a privacy issue, since even if each model is private, the very choice of the hyperparameter can break privacy. Therefore, privatizing this approach requires employing a private selection algorithm over the hyperparameter, as studied by \citet{liu2019private,papernot2022hyperparameter}.

A second approach to hyperparameter tuning is based on differentiable programming: instead of predetermining a finite set of possible values for $\omega$, we can solve \eqref{eq: hyperparameter bilevel} using BO methods, namely ``differentiate through $\omega$'', and update it accordingly on the fly.
This approach was pioneered by \citet{bengio2000gradient}, and was further developed by multiple works \citep{maclaurin2015gradient,franceschi2017forward,franceschi2018bilevel,lorraine2020optimizing,engstrom2025optimizing}. However, we are not aware of a private variant
of this methodology, and our goal is to provide one here.

Applying our algorithm to this problem turns out to carry some useful simplifications. Note that $\nabla_{\omega}f=0$, $\nabla_{\omega}g(\omega,\theta)= R(\theta)$, thus letting $\theta_t\approx \arg\min_{\bw}g(\omega_t,\theta)$ and $\theta_t^\lambda\approx \arg\min_{\theta}(f(\omega_t,\theta)+ \lambda g(\omega_t,\theta))$
be private solutions to the inner problems,\footnote{Note that $\theta_t$ can be interpreted as a regularized model over the training data,
while $\theta_t^\lambda$ simply corresponds to a model trained by mixing in a bit of the validation loss (weighted down by $1/\lambda$).} our algorithm yields the following outer loop:
\begin{align*}
\omega_{t+1}
&=\arg\min_{u\geq 0}
\Big\{u \left[\nabla_\omega f(\omega_t,\theta_t^\lambda)+\lambda( \nabla_\omega g(\omega_t,\theta_t^\lambda)- \nabla_\omega g(\omega_t,\theta_t))+\Ncal(0,\sigma^2)\right]+\tfrac{1}{2\eta}(u-\omega_t)^2\Big\} \nonumber
\\
&=\arg\min_{u\geq 0}
\left\{u\left[0+\lambda( R(\theta_t^\lambda)- R(\theta_t))+\Ncal(0,\sigma^2)\right]+\tfrac{1}{2\eta}(u-\omega_t)^2\right\} \nonumber
\\&=\max\left\{0,~\omega_t+\eta\lambda\cdot\Ncal\left(R(\theta_t^\lambda)-R(\theta_t),\,\sigma^2/\lambda^2\right)\right\}~,
\end{align*}
where the last equality is easily verified by solving a one-dimensional quadratic over a half-line.

Overall, we obtain a simple closed-form update rule for privately tuning the regularization parameter $\omega$, based on two private training subroutines (which can be solved in a black-box manner),
according to a normal variable centered at the difference between their associated complexities measured by $R$.

\section{Proofs} \label{sec: proofs}

Throughout the proof section, we abbreviate $f_i(\cdot)=f(\cdot\,;\xi_i),~g_i(\cdot)=g(\cdot\,;\xi_i),~F=F_S$.

\subsection{Proof of Lemma~\ref{lem: inexact grad}} \label{sec: proof of fullb inexact grad}

Note that the two sub-problems solved by Algorithm~\ref{alg: main} are strongly-convex and admit Lipschitz components over
$\Ycal;~g(\bx,\cdot)$ by assumption, and $f+\lambda g$ by combining this with the smoothness/Lipschitzness of $f$, as follows:
\begin{lemma} \label{lem: sc and lip}
If $\lambda\geq \max\{\frac{2L_1^g}{\mu_g},\frac{L_0^f}{L_0^g}\}$ then for all $\bx\in\Xcal:~f(\bx,\cdot)+\lambda g(\bx,\cdot)$ is $\frac{\lambda \mu_g}{2}$ strongly-convex, and moreover for all $i\in[n]:~f_i(\bx,\cdot)+\lambda g_i(\bx,\cdot)$ is $2\lambda L_0^g$-Lipschitz.
\end{lemma}

We therefore invoke the following guarantee, which provides the optimal result for strongly-convex DP ERM via DP-Loc-GD (Algorithm~\ref{alg: DPLocGD}).

\begin{restatable}{theorem}{DPGD}
\label{thm: DPGD}
Suppose that $h:\reals^{d_y}\to\reals$ is a $\mu$-strongly-convex function of the form $h(\by)=\frac{1}{n}\sum_{i=1}^{n}h(\by,\xi_i)$ where
$h(\cdot,\xi_i)$ is $L$-Lipschitz for all $i\in[n]$.
Suppose $\arg\min h=:\by^*\in\mathbb{B}(\by_0,R_0)$ and that $n\geq\frac{L R_0^{2/\log(d_y)}}{\mu\epsilon'}$.
Then there is an assignment of parameters
$M=\log_2\log(\frac{\mu\epsilon'n}{L})\,,
~\sigma^2_{\mathrm{GD}}=\widetilde{\Ocal}(L^2/\epsilon'^2)\,,~\eta_t=\frac{1}{\mu (t+1)},~T_{\mathrm{GD}}=n^2\,,~
R_m=\widetilde{\Theta}\left(\sqrt{\frac{R_{m-1}L}{\mu\epsilon'n}}+\frac{L\sqrt{d_y}}{\mu\epsilon'n}\right)
$ such that
running Algorithm~\ref{alg: DPLocGD} satisfies $(\epsilon',\delta')$-DP, and outputs $\by_{\out}$ such that $\norm{\by_{\out}-\by^*}=\widetilde{\Ocal}\left(
\frac{L\sqrt{d_y}}{\mu n\epsilon'}
\right)$ with probability at least $1-\gamma$.
\end{restatable}

Although the rate in Theorem~\ref{thm: DPGD} appears in prior works such as \citep{bassily2014private,feldman2020private}, it is typically manifested through a bound in expectation (and in terms of function value) as opposed to with high probability, required for our purpose. We therefore, for the sake of completeness, provide a self-contained proof of
Theorem~\ref{thm: DPGD} in Appendix~\ref{sec: DPGD}.

Applied to the functions $g(\bx_t,\cdot)$ and $f(\bx_t,\cdot)+\lambda g(\bx_t,\cdot)$, and invoking
Lemma~\ref{lem: sc and lip}, yields the following.

\begin{corollary} \label{cor: Bassily dist}
If $\lambda\geq\max\{\frac{2L_1^g}{\mu_g},\frac{L_0^f}{L_0^g}\}$, then
$\tby_t$ and $\tby_t^\lambda$ (as appear in Algorithm~\ref{alg: main}) satisfy with probability at least $1-\gamma:$
\[
\max\left\{\norm{\tby_t-\by^*(\bx_t)},~\snorm{\tby_t^\lambda-\by^\lambda(\bx_t)}\right\}=\widetilde{\Ocal}\left(\frac{L_0^g\sqrt{d_y T}}{\epsilon\mu_g n}\right)~.
\]
\end{corollary}

We are now ready to prove the main proposition of this section, which we restate below:

\propInexact*

\begin{proof}[{Proof of Lemma~\ref{lem: inexact grad}}]
As in \eqref{eq: first order}, we note that by construction $\Lcal_\lambda^*(\bx)=\arg\min_{\by}\Lcal_\lambda(\bx,\by)$, therefore it holds that
\begin{align*}
\nabla\Lcal_\lambda^*(\bx)=\nabla_x f(\bx,\by^\lambda(\bx))+\lambda\left(\nabla_x g(\bx,\by^\lambda(\bx))-\nabla_{x}g(\bx,\by^*(\bx))\right)~.
\end{align*}
Denoting $\bg_t=\nabla_x f(\bx_t,\tby^\lambda_t)+\lambda\left(\nabla_x g(\bx_t,\tby^\lambda_t)-\nabla_x g(\bx_t,\tby_t)\right)$, by the smoothness of $f$ and $g$
we see that
\[
\norm{\bg_t-\nabla\Lcal_\lambda^*(\bx_t)}
\leq L_1^f\norm{\tby_t^\lambda-\by^\lambda(\bx_t)}+\lambda L_1^g\norm{\tby_t^\lambda-\by^\lambda(\bx_t)}+\lambda L_1^g\norm{\tby_t-\by^*(\bx_t)}~.
\]
Applying Corollary~\ref{cor: Bassily dist} and union bounding over $T$, we can further bound the above as
\begin{align*}
\norm{\bg_t-\nabla\Lcal_\lambda^*(\bx_t)}
&=\widetilde{\Ocal}\left(\frac{L_1^f L_0^g\sqrt{d_y T}}{\epsilon\mu_g n}
+\frac{\lambda L_1^g L_0^g\sqrt{d_y T}}{\epsilon\mu_g n}+\frac{\lambda L_1^g L_0^g\sqrt{d_y T}}{\epsilon\mu_g n}\right)
\\
&=\widetilde{\Ocal}\left( \frac{\lambda L_1^g L_0^g\sqrt{d_y T}}{\epsilon\mu_g n}\right)
\\
&=\widetilde{\Ocal}\left( \frac{\lambda \ell\kappa\sqrt{d_y T}}{\epsilon n}\right)
~,
\end{align*}
where the second bound holds for $\lambda\geq\frac{L_1^f}{L_1^g}$ .
\end{proof}

\subsection{Proof of Lemma~\ref{lem: Li improved Lip}}

We start by providing two lemmas, both of which 
borrow ideas that appeared in the smoothness analysis of \citet{chen2024finding}.

\begin{restatable}{lemma}{ylamLip}
\label{lem: ylam Lip}
$\by^\lambda(\bx):\reals^{d_x}\to\reals^{d_y}$ is $\left(\frac{4 L_1^g}{\mu_g}\right)$-Lipschitz.
\end{restatable}

\begin{proof}[Proof of Lemma~\ref{lem: ylam Lip}]
Differentiating $\nabla_{y}\Lcal_\lambda(\bx,\by^\lambda(\bx))=\bzero$ with respect the first argument gives
\[
\nabla^2_{xy}\Lcal_\lambda(\bx,\by^\lambda(\bx))+\nabla\by^\lambda(\bx)\cdot\nabla_{yy}\Lcal_\lambda(\bx,\by^\lambda(\bx))=\bzero~,
\]
hence
\[
\nabla\by^\lambda(\bx)=-\nabla^2_{xy}\Lcal_\lambda(\bx,\by^\lambda(\bx))\cdot\left[\nabla_{yy}\Lcal_\lambda(\bx,\by^\lambda(\bx))\right]^{-1}~.
\]
Noting that $\nabla^2_{xy}\Lcal_{\lambda}\preceq2\lambda L_1^g$ and $\nabla^2_{yy}\Lcal_{\lambda}\succeq\lambda\mu_g/2$ everywhere, hence $[\nabla^2_{yy}\Lcal_{\lambda}]^{-1}\preceq2/\lambda\mu_g$
we get that
\[
\norm{\nabla\by^\lambda(\bx)}\leq
2\lambda L_1^g\cdot\frac{2}{\lambda\mu_g}
=\frac{4 L_1^g}{\mu_g}~.
\]
\end{proof}

\begin{restatable}{lemma}{ylamclose}
\label{lem: ylam close}
For all $\bx\in\Xcal:~\norm{\by^\lambda(\bx)-\by^*(\bx)}\leq\frac{L_0^f}{\lambda\mu_g}$.
\end{restatable}

\begin{proof}[Proof of Lemma~\ref{lem: ylam close}]
First, note that by definition of $\by^\lambda(\bx)$ it holds that
\[
\bzero=\nabla_y\Lcal_\lambda(\bx,\by^\lambda(\bx))
=\nabla_{y}f(\bx,\by^\lambda(\bx))+\lambda \nabla_{y}g(\bx,\by^\lambda(\bx))~,
\]
hence
\[
\nabla_{y}g(\bx,\by^\lambda(\bx))=-\frac{1}{\lambda}\cdot\nabla_{y}f(\bx,\by^\lambda(\bx))~,
\]
so in particular by the Lipschitz assumption on $f$ we see that
\[
\norm{\nabla_{y}g(\bx,\by^\lambda(\bx))}
\leq\frac{L_0^f}{\lambda}~.
\]
By invoking the $\mu$-strong convexity of $g$ we further get
\begin{align*}
\norm{\by^\lambda(\bx)-\by^*(\bx)}
\leq~&\frac{1}{\mu}\norm{\nabla_{y}g(\bx,\by^\lambda(\bx))-\underset{=\bzero}{\underbrace{\nabla_{y}g(\bx,\by^*(\bx))}}}
\leq\frac{L_0^f}{\lambda\mu}~.
\end{align*}

\end{proof}

We are now ready to prove the main result of this section, restated below.

\LiImprovedLip*

\begin{proof}[Proof of Lemma~\ref{lem: Li improved Lip}]
For all $\bx\in\Xcal$ it holds that
\begin{align}
\norm{\nabla \Lcal^*_{\lambda,i}(\bx)}
&=\norm{\nabla_x f_i(\bx,\by^\lambda(\bx))+\lambda\left[\nabla_x g_i(\bx,\by^\lambda(\bx))-\nabla_x g_i(\bx,\by^*(\bx))\right]}\nonumber
\\
&\leq \norm{\nabla_x f_i(\bx,\by^\lambda(\bx))}
+\lambda\norm{\nabla_x g_i(\bx,\by^\lambda(\bx))-\nabla_x g_i(\bx,\by^*(\bx))}~, \label{eq: Lip 2 summands}
\end{align}
thus we will bound each of the summands above.

For the first term, since $\by^\lambda$ is $\frac{4L_1^g}{\mu_g}$ Lipschitz according to Lemma~\ref{lem: ylam Lip}, and $f_i$ is $L_0^f$-Lipschitz by assumption,
the chain rule yields the bound
\begin{equation} \label{eq: 1st lip term}
\norm{\nabla_x f_i(\bx,\by^\lambda(\bx))}
\leq \frac{4L_1^g L_0^f}{\mu_g}
\leq 4\ell\kappa
~.
\end{equation}
As to the second term, since $g_i$ is $L_1^g$-smooth, we use Lemma~\ref{lem: ylam close} and get that
\begin{equation}\label{eq: 2nd lip term}
\lambda\norm{\nabla_x g_i(\bx,\by^\lambda(\bx))-\nabla_x g_i(\bx,\by^*(\bx))}
\leq \lambda L_1^g\norm{\by^\lambda(\bx)-\by^*(\bx)}
\leq \frac{L_1^g L_0^f}{\mu_g}
\leq\ell\kappa~.
\end{equation}
Plugging Eqs. (\ref{eq: 1st lip term}) and (\ref{eq: 2nd lip term}) into \eqref{eq: Lip 2 summands} completes the proof.

\end{proof}

\subsection{Proof of Proposition~\ref{prop: outer alg}}

As $\nu_0,\dots,\nu_{T-1}\overset{iid}{\sim}\Ncal(0,\sigma^2 I)$, a standard Gaussian norm bound (cf. \citealt[Theorem 3.1.1]{vershynin2018high}) ensures that with probability at least
$1-\gamma$, for all $t\in\{0,1,\dots,T-1\}:~\norm{\nu_t}^2\lesssim d\sigma^2\log(T/\gamma)\lesssim\frac{\alpha^2}{64}$.
We therefore condition the rest of the proof on the highly probable event under which this uniform norm bound indeed holds.
We continue by introducing some notation. 
We denote $\tnab_t=\tnab{h}(\bx_t)+\nu_t$, and
$\delta_t:=\tnab_t-\nabla h(\bx_t)$. We further denote
\begin{align*}
\bx_t^+&:=\arg\min_{\bu\in\Xcal}\left\{\inner{\nabla h(\bx_t),\bu}+\frac{1}{2\eta}\norm{\bx_{t}-\bu}^2\right\}~,
\\
\Gcal_t&:=\frac{1}{\eta}(\bx_t-\bx_t^+)~,
\\
\rho_t&:=\frac{1}{\eta}(\bx_t-\bx_{t+1})
~.
\end{align*}
Note that by construction, 
\[
\Gcal_t =\Gcal_{h,\eta}(\bx_t):=\frac{1}{\eta}\left(\bx_t-\Pcal_{\nabla h,\eta}(\bx_t)\right)~,
~~~~~
\Pcal_{\nabla h,\eta}(\bx_t):=\arg\min_{\bu\in\Xcal}\left[\inner{\nabla h(\bx_t),\bu}+\frac{1}{2\eta}\norm{\bu-\bx_t}^2\right]~,
\]
and that $\Gcal_{t_\out}$ is precisely the quantity we aim to bound.
We start by proving some useful lemmas regarding the quantities defined above.

\begin{lemma} \label{lem: small delta}
Under the event that $\norm{\nu_t}^2\leq\frac{\alpha^2}{64}$ for all $t$, it holds that $\norm{\delta_t}\leq \frac{\alpha}{4}$.
\end{lemma}

\begin{proof}
    By our assumptions on $\beta,\norm{\nu_t}$, we get that
    \[
    \norm{\delta_t}\leq \snorm{\tnab h(\bx_t)-\nabla h(\bx_t)}+\norm{\nu_t}\leq \beta+\frac{\alpha}{8}\leq\frac{\alpha}{4}~.
    \]
\end{proof}

\begin{lemma} \label{lem: rho 2}
    It holds that $\sinner{\tnab_t,\rho_t}\geq \norm{\rho_t}^2$.
\end{lemma}

\begin{proof}
By definition, $\bx_{t+1}=\arg\min_{\bu\in\Xcal}
\left\{\sinner{\tnab_t,\bu}+\frac{1}{2\eta}\norm{\bx_{t}-\bu}^2\right\}$. Hence, by the first-order optimality criterion, for any $\bu\in\Xcal:$
\[
\inner{\tnab_t+\frac{1}{\eta}(\bx_{t+1}-\bx_t),\bu-\bx_{t+1}}\geq 0~.
\]
In particular, setting $\bu=\bx_{t}$
yields
\[
0\leq \inner{\tnab_t+\frac{1}{\eta}(\bx_{t+1}-\bx_t),\bx_t-\bx_{t+1}}
=\inner{\tnab_t-\rho_t,\eta\rho_t}
=\eta\left(\inner{\tnab_t,\rho_t}-\norm{\rho_t}^2\right)~,
\]
which proves the claim since $\eta>0$.
\end{proof}

With the lemmas above in hand, we are now ready to prove Proposition~\ref{prop: outer alg}.
Note that by construction, the algorithm returns the index $t$ with minimal $\norm{\rho_t}$.
Further note that $\norm{\rho_t-\Gcal_t}\leq\norm{\delta_t}$ by  Lemma~\ref{lem: G is Lip}, thus
\begin{equation} \label{eq: G rho bound}
\norm{\Gcal_t}\leq \norm{\rho_t}+\norm{\delta_t}\leq \norm{\rho_t}+\frac{\alpha}{4}~,
\end{equation}
where the last inequality is due to Lemma~\ref{lem: small delta},
hence it suffices to bound $\norm{\rho_{t_\out}}$ (which is the quantity measured by the output rule).
To that end, by smoothness, we have for any $t\in\{0,1,\dots,T-2\}:$
\begin{align*}
h(\bx_{t+1})&\leq h(\bx_t)+\inner{\nabla h(
\bx_t),\bx_{t+1}-\bx_t}+\frac{L}{2}\norm{\bx_{t+1}-\bx_t}^2
\\
&=h(\bx_t)-\eta\inner{\nabla h(\bx_t),\rho_t}+\frac{L\eta^2}{2}\norm{\rho_t}^2
\\
&= h(\bx_t)-\eta\inner{\tnab_t,\rho_t}+\frac{L\eta^2}{2}\norm{\rho_t}^2+\eta\inner{\delta_t,\rho_t}
\\
&\leq h(\bx_t)-\eta\norm{\rho_t}^2+\frac{L\eta^2}{2}\norm{\rho_t}^2+\eta\norm{\delta_t}\cdot\norm{\rho_t}~,
\end{align*}
where the last inequality followed by applying Lemma~\ref{lem: rho 2} and Cauchy-Schwarz. Rearranging, and recalling that $\eta=\frac{1}{2L}$, hence $1<2-L\eta$
and also $\frac{1}{\eta}=2L$,
we get that
\[
\norm{\rho_t}^2-2\norm{\delta_t}\cdot\norm{\rho_t}
\leq
\left(2-L\eta\right)\norm{\rho_t}^2-2\norm{\delta_t}\cdot\norm{\rho_t}
\leq \frac{2\left(h(\bx_t)-h(\bx_{t+1})\right)}{\eta}
=4L\left(h(\bx_t)-h(\bx_{t+1})\right)~.
\]
Summing over $t\in\{0,1\dots,T-1\}$, using the telescoping property of the right hand side, and dividing by $T$ gives that
\begin{equation} \label{eq: sum rho}
\frac{1}{T}\sum_{t=0}^{T-1}\norm{\rho_t}\left(\norm{\rho_t}-2\norm{\delta_t}\right)\leq \frac{4L\left(h(\bx_0)-\inf h\right)}{T}~.
\end{equation}
Note that if for some $t\in\{0,1,\dots,T-1\}:~\norm{\rho_t}\leq\frac{3\alpha}{4}$ then we have proved our desired claim by \eqref{eq: G rho bound}
and the fact that $\norm{\rho_{t_\out}}=\min_t\norm{\rho_t}$ by definition. On the other hand, assuming that $\norm{\rho_t}>\frac{3\alpha}{4}$ for all $t$, invoking Lemma~\ref{lem: small delta}, we see that $\norm{\rho_t}-2\norm{\delta_t}\geq \norm{\rho_t}-\frac{\alpha}{2}\geq\frac{1}{3}\norm{\rho_t}$, which implies $\norm{\rho_t}\left(\norm{\rho_t}-2\norm{\delta_t}\right)\geq\frac{1}{3}\norm{\rho_t}^2$. Combining this with \eqref{eq: sum rho} yields
\[
\norm{\rho_{t_\out}}^2=
\min_{t\in\{0,1,\dots,T-1\}}\norm{\rho_t}^2
\leq\frac{1}{T}\sum_{t=0}^{T-1}\norm{\rho_t}^2\leq \frac{12L\left(h(\bx_0)-\inf h\right)}{T}~,
\]
and the right side is bounded by $\frac{9\alpha^2}{16}$ for $T=\Ocal\left(\frac{L(h(\bx_0)-\inf h)}{\alpha^2}\right)$, finishing the proof by \eqref{eq: G rho bound}.

\subsection{Proof of Theorem~\ref{thm: main}}

We start by proving the privacy guarantee. Since $\Lcal^*_{\lambda,i}$ is
$\Ocal(\ell\kappa)$-Lipschitz by Lemma~\ref{lem: Li improved Lip}, the sensitivity of $\nabla_x f(\bx_t,\tby^\lambda_t)+\lambda\left(\nabla_x g(\bx_t,\tby^\lambda_t)-\nabla_x g(\bx_t,\tby_t)\right)$ is at most $\Ocal(\ell\kappa)$.
Hence, by setting $\sigma^2=C\frac{\ell^2\kappa^2\log(T/\delta)T}{\epsilon^2 n^2 }$ for a sufficiently large absolute constant $C>0$, $\tbg_t$ is $(\frac{\epsilon}{\sqrt{18T}},\frac{\delta}{3(T+1)})$-DP. By basic composition, since each iteration also runs $(\frac{\epsilon}{\sqrt{18T}},\frac{\delta}{3(T+1)})$-DP-Loc-GD twice,
we see that each iteration of the algorithm is $(3\cdot\frac{\epsilon}{\sqrt{18T}},3\cdot\frac{\delta}{3(T+1)})=(\frac{\epsilon}{\sqrt{2T}},\frac{\delta}{(T+1)})$-DP.
Noting that under our parameter assignment $\frac{\epsilon}{\sqrt{T}}\ll 1$, by advanced composition we get that throughout $T$ iterations, the algorithm is overall $(\epsilon,\delta)$-DP as claimed.

We turn to analyze the utility of the algorithm. It holds that
\begin{align}
\norm{\Gcal_{ F,\eta}(\bx_{t_\out})}
&\leq\norm{\Gcal_{ F,\eta}(\bx_{t_\out})-\Gcal_{\Lcal^*_\lambda,\eta}(\bx)}+\norm{\Gcal_{\Lcal^*_\lambda,\eta}(\bx_{t_\out})} \nonumber
\\&\leq \snorm{\nabla  F(\bx_{t_\out})-\nabla\Lcal^*_\lambda(\bx)}+\snorm{\Gcal_{\Lcal^*_\lambda,\eta}(\bx_{t_\out})} \nonumber
\\&\lesssim \frac{\ell\kappa^3}{\lambda}+\snorm{\Gcal_{\Lcal^*_\lambda,\eta}(\bx_{t_\out})} \nonumber
\\&\leq
\frac{\alpha}{2}+\snorm{\Gcal_{\Lcal^*_\lambda,\eta}(\bx_{t_\out})}
~, \label{eq: Gout bound}
\end{align}
where the second inequality is due to Lemma~\ref{lem: G is Lip}, the third due to Lemma~\ref{lem: kwon}.b, and the last by our assignment of $\lambda$. It therefore remains to bound $\snorm{\Gcal_{\Lcal^*_\lambda,\eta}(\bx_{t_\out})}$.

To that end, applying Proposition~\ref{prop: outer alg} to the function $h=\Lcal^*_\lambda$,
under our assignment of $T$ --- which accounts for the smoothness and initial sub-optimality bounds ensured by Lemma~\ref{lem: kwon} ---
we get that $\snorm{\Gcal_{\Lcal^*_\lambda,\eta}(\bx_{t_\out})}\leq\frac{\alpha}{2}$, for $\alpha$ as small as

\begin{equation} \label{eq: alpha how small}
   \alpha
    =\Theta\left(\max\{\beta,\sigma\sqrt{d_x \log(T/\gamma)}\}\right)
\end{equation}
By Lemma~\ref{lem: inexact grad}, it holds that 

\begin{equation} \label{eq: final alpha bound 1}
\beta
=\widetilde{\Ocal}\left(\frac{\lambda\ell\kappa\sqrt{d_y T}}{\epsilon n}\right)
=\widetilde{\Ocal}\left(\frac{\ell^{5/2}\kappa^{11/2}\Delta_F^{1/2}\sqrt{d_y}}{\alpha^2 \epsilon n}\right)
\end{equation}
and we also have
\begin{equation} \label{eq: final alpha bound 2}
\sigma \sqrt{d_x \log(T/\gamma)}
=\widetilde{\Ocal}\left(\frac{\ell^{3/2}\kappa^{5/2}\Delta_F^{1/2}\sqrt{d_x}}{\alpha\epsilon n}\right)
~. 
\end{equation}
Plugging (\ref{eq: final alpha bound 1}) and (\ref{eq: final alpha bound 2}) back into \eqref{eq: alpha how small}, and solving for $\alpha$, completes the proof.

\subsection{Proof of Theorem~\ref{thm: main minib}}
\label{sec: proof minib}

Throughout this section, we abbreviate $b=b_{\mathrm{out}}$.
We will need the following lemma, which is the
mini-batch analogue of Lemma~\ref{lem: inexact grad} from the full-batch setting.

\begin{restatable}{lemma}{propInexactMinib}\label{lem: inexact grad minib}
If $\lambda\geq\max\{\frac{2L_1^g}{\mu_g},\frac{L_0^f}{L_0^g},\frac{L^f_1}{L^g_1}\}$, then there is $\beta_b=\widetilde{\Ocal}\left(\frac{\lambda \ell\kappa\sqrt{d_y T}}{\epsilon n}+\frac{\ell\kappa}{ b}\right)$ such that with probability at least $1-\gamma/2$,
$\bg_t^B:=\nabla_x f(\bx_t,\tby^\lambda_t;B_t)+\lambda\left(\nabla_x g(\bx_t,\tby^\lambda_t;B_t)-\nabla_x g(\bx_t,\tby_t;B_t)\right)$
satisfies for all $t\in\{0,\dots,T-1\}:~\|\nabla\Lcal^*_\lambda(\bx_t)-\bg^B_t\|\leq \beta_b$.
\end{restatable}

\begin{proof}[Proof of Lemma~\ref{lem: inexact grad minib}]
It holds that
\[
\norm{\nabla\Lcal^*_\lambda(\bx_t)-\bg^B_t}
\leq \norm{\nabla\Lcal^*_\lambda(\bx_t)-\E[\bg^B_t]}
+\norm{\bg^B_t-\E[\bg^B_t]}~.
\]
To bound the first summand, note that $\E[\bg^B_t]=
\nabla_x f(\bx_t,\tby^\lambda_t)+\lambda\left(\nabla_x g(\bx_t,\tby^\lambda_t)-\nabla_x g(\bx_t,\tby_t)\right)$,
and therefore with probability at least $1-\gamma/4:$
\[
\norm{\nabla\Lcal^*_\lambda(\bx_t)-\E[\bg^B_t]}=\Ocal\left(\frac{\lambda L_1^g L_0^g\sqrt{d_y T}}{\epsilon\mu_g n}\right)
=\Ocal\left(\frac{\lambda\ell\kappa\sqrt{d_y T}}{\epsilon n}\right)
~,
\]
following the same proof as Lemma~\ref{lem: inexact grad} in Section~\ref{sec: proof of fullb inexact grad},
by replacing Theorem~\ref{thm: DPGD} by the mini-batch Theorem~\ref{thm: DPLocSGD} (whose guarantee holds regardless of the inner batch size) .

To bound the second summand, note that $\|\nabla_x f(\bx_t,\tby^\lambda_t;\xi)+\lambda(\nabla_x g(\bx_t,\tby^\lambda_t;\xi)-\nabla_x g(\bx_t,\tby_t;\xi))\|\leq M=\Ocal(\ell\kappa)$ for every $\xi\in\Xi$, by Lemma~\ref{lem: Li improved Lip}. Hence, $\bg^B_t$ is the average of $b$
independent vectors bounded by $M$, all with the same mean, and therefore a standard concentration bound (cf. \citealt{jin2019short}) ensures that $\|\bg_t^B-\E[\bg_t^B]\|=\widetilde{\Ocal}(M/b)$ with probability at least $1-\gamma/4$, which completes the proof.
\end{proof}

We can now prove the main mini-batch result:

\begin{proof}[{Proof of Theorem~\ref{thm: main minib}}]
We start by proving the privacy guarantee. Since $\Lcal^*_{\lambda,i}$ is
$\Ocal(\ell\kappa)$-Lipschitz by Lemma~\ref{lem: Li improved Lip}, the sensitivity of $\nabla_x f(\bx_t,\tby^\lambda_t;B_t)+\lambda\left(\nabla_x g(\bx_t,\tby^\lambda_t;B_t)-\nabla_x g(\bx_t,\tby_t;B_t)\right)$ is at most
$\Ocal(\ell\kappa)$.
Accordingly, the ``unamplified'' Gaussian mechanism (Lemma~\ref{lem: gaussian mech}) ensures $(\tilde{\epsilon},\tilde{\delta})$-DP for $\tilde{\epsilon}=\widetilde{\Theta}\left(\frac{\ell\kappa}{b\sigma}\right)
$,
and hence is amplified (Lemma~\ref{lem: amplification}) to $(\epsilon_0,\delta_0)$-DP for $\epsilon_0=\widetilde{\Theta}\left(\frac{\ell\kappa}{b\sigma}\cdot \frac{b}{n}\right)=\frac{\epsilon}{\sqrt{18T}}$,
the last holding for sufficiently large $\sigma^2=\widetilde{\Theta}\left(\frac{\ell^2\kappa^2 T}{\epsilon^2n^2}\right)$,
and for $\delta_0=\frac{\delta}{3(T+1)}$.
Therefore, basic composition shows that each iteration of the algorithm is $(3\cdot\frac{\epsilon}{\sqrt{18T}},3\frac{\delta}{3(T+1)})=(\frac{\epsilon}{\sqrt{2T}},\frac{\delta}{T+1})$-DP.
Since $\epsilon/\sqrt{T}\ll 1$ under our parameter assignment, advanced composition over the $T$ iterations yields the $(\epsilon,\delta)$-DP guarantee.

We turn to analyze the utility of the algorithm. It holds that
\begin{align}
\norm{\Gcal_{ F,\eta}(\bx_{t_\out})}
&\leq\norm{\Gcal_{ F,\eta}(\bx_{t_\out})-\Gcal_{\Lcal^*_\lambda,\eta}(\bx)}+\norm{\Gcal_{\Lcal^*_\lambda,\eta}(\bx_{t_\out})} \nonumber
\\&\leq \snorm{\nabla  F(\bx_{t_\out})-\nabla\Lcal^*_\lambda(\bx)}+\snorm{\Gcal_{\Lcal^*_\lambda,\eta}(\bx_{t_\out})} \nonumber
\\&\lesssim \frac{\ell\kappa^3}{\lambda}+\snorm{\Gcal_{\Lcal^*_\lambda,\eta}(\bx_{t_\out})} \nonumber
\\&\leq
\frac{\alpha}{2}+\snorm{\Gcal_{\Lcal^*_\lambda,\eta}(\bx_{t_\out})}
~, \label{eq: Gout bound minib}
\end{align}
where the second inequality is due to Lemma~\ref{lem: G is Lip}, the third due to Lemma~\ref{lem: kwon}, and the last by our assignment of $\lambda$. It therefore remains to bound $\snorm{\Gcal_{\Lcal^*_\lambda,\eta}(\bx_{t_\out})}$.

To that end, applying Proposition~\ref{prop: outer alg} to the function $h=\Lcal^*_\lambda$,
under our assignment of $T$ --- which accounts for the smoothness and initial sub-optimality bounds ensured by Lemma~\ref{lem: kwon} ---
we get that $\snorm{\Gcal_{\Lcal^*_\lambda,\eta}(\bx_{t_\out})}\leq\frac{\alpha}{2}$, for $\alpha$ as small as

\begin{equation} \label{eq: alpha how small minib}
    \alpha
    =\Theta\left(\max\{\beta_b,\sigma\sqrt{d_x \log(T/\gamma)}\}\right)
\end{equation}
By Lemma~\ref{lem: inexact grad minib}, it holds that 

\begin{equation}\label{eq: final alpha bound minib1}
\beta_b
=\widetilde{\Ocal}\left(\frac{\lambda \ell\kappa\sqrt{d_y T}}{\epsilon n}+\frac{\ell\kappa}{ b}\right)
=\widetilde{\Ocal}\left(\frac{\ell^{5/2}\kappa^{11/2}\Delta_F^{1/2}\sqrt{d_y}}{\alpha^2 \epsilon n}+\frac{\ell\kappa}{ b}\right)
~,
\end{equation}
and we also have
\begin{equation}\label{eq: final alpha bound minib2}
\sigma \sqrt{d_x \log(T/\gamma)}
=\widetilde{\Ocal}\left(\frac{\ell^{3/2}\kappa^{5/2}\Delta_F^{1/2}\sqrt{d_x}}{\alpha\epsilon n}\right)~.
\end{equation}
Plugging (\ref{eq: final alpha bound minib1}) and (\ref{eq: final alpha bound minib2}) back into \eqref{eq: alpha how small minib}, and solving for $\alpha$, completes the proof.
    
\end{proof}

\subsection{Proof of Theorem~\ref{thm: population}}

The proof is based on the following lemma that establishes uniform convergence of hypergradients.

\begin{lemma} \label{lem: uc}
Suppose $\overline{\Xcal}\subset\reals^d_x$ is a subset of bounded diameter
$\diam(\overline{\Xcal})\leq D$, and that $S\sim\Pcal^{n}$. Then with probability at least $1-\gamma$ for all $\bx\in\overline{\Xcal}:$
\begin{align*}
\norm{\nabla F_{\Pcal}(\bx)-\nabla F_S(\bx)}=\widetilde{\Ocal}\bigg(&\frac{(\ell^2\kappa^2+\ell\kappa)\sqrt{d_x\log(D/\gamma)}}{\sqrt{n}}
+\frac{(\ell^2\kappa^3+\ell\kappa^2)\sqrt{\log(1/\gamma)}}{\sqrt{n}}
\\&+\frac{(\ell^{3/2}\kappa^{5/2}+\ell^{1/2}\kappa^{3/2})\log^{1/4}(1/\gamma)}{n^{1/4}}
\bigg)~.
\end{align*}
\end{lemma}

\begin{proof}[Proof of Lemma~\ref{lem: uc}]
We start by setting up notation.
Given a dataset $S\sim\Pcal^n$, recall that we denote by $F_\Pcal/F_S$ the stochastic/empirical hyperobjectives, by $f_\Pcal/f_S$ and $g_\Pcal/g_S$ the stochastic/empirical upper- and lower-level objectives.
We denote by $\by^*_\Pcal(\bx)=\arg\min_{\by}g_\Pcal(\bx,\by)$ and $\by^*_S(\bx)=\arg\min_{\by}g_S(\bx,\by)$. It holds that
\begin{align*}
\nabla F_\Pcal(\bx)&=\nabla_x f_\Pcal(\bx,\by^*_\Pcal(\bx))+
(\frac{d\by^*_\Pcal(\bx)}{d\bx})^\top\nabla_y f_\Pcal(\bx,\by^*_\Pcal(\bx))
\\
\nabla F_S(\bx)&=\nabla_x f_S(\bx,\by^*_S(\bx))+
(\frac{d\by^*_S(\bx)}{d\bx})^\top\nabla_y f_S(\bx,\by^*_S(\bx))~.
\end{align*}
Thus,
\begin{align*}
\|\nabla F_\Pcal(\bx)-\nabla F_S(\bx)\|
&\leq \|\nabla_x f_\Pcal(\bx,\by^*_\Pcal(\bx))-\nabla_x f_S(\bx,\by^*_S(\bx))\|
\\&~~~+\|(\frac{d\by^*_\Pcal(\bx)}{d\bx})^\top\nabla_y f_\Pcal(\bx,\by^*_\Pcal(\bx))
-(\frac{d\by^*_S(\bx)}{d\bx})^\top\nabla_y f_S(\bx,\by^*_S(\bx))\|
\\
&\leq \underset{(1)}{\underbrace{\|\nabla_x f_\Pcal(\bx,\by^*_\Pcal(\bx))-\nabla_x f_S(\bx,\by^*_S(\bx))\|}}
\\&~~~+\underset{(2)}{\underbrace{\|\frac{d\by^*_\Pcal(\bx)}{d\bx}\|\cdot \|\nabla_y f_\Pcal(\bx,\by^*_\Pcal(\bx))-\nabla_y f_S(\bx,\by^*_S(\bx))\|}}
\\&~~~+\underset{(3)}{\underbrace{\|\nabla_y f_S(\bx,\by^*_S(\bx))\|\cdot \|\frac{d\by^*_\Pcal(\bx)}{d\bx}-\frac{d\by^*_S(\bx)}{d\bx}\|}}~,
\end{align*}
where the last inequality used the fact that $|ab-cd|=|ab-ad+ad-cd|\leq |a||b-d|+|d||a-c|$.
We turn to bound each summand with high probability, which immediately results in Lemma~\ref{lem: uc} by summing and union bounding.

\paragraph{Bound $(1)$.}
By smoothness, it holds that
\[
\|\nabla_x f_\Pcal(\bx,\by^*_\Pcal(\bx))-\nabla_x f_S(\bx,\by^*_S(\bx))\|
\leq \|\nabla_x f_\Pcal(\bx,\by^*_\Pcal(\bx))-\nabla_x f_S(\bx,\by^*_\Pcal(\bx))\|+L_1^f\|\by^*_\Pcal(\bx)-\by^*_S(\bx)\|~.
\]
To bound the first term, we note that the partial derivatives are bounded by $L_0^f$ by Lipschitzness. Hence, we can apply a uniform convergence bound due to \citet[Theorem 1]{mei2018landscape} and get with probability at least $1-\gamma/9:$
\begin{align*}
\|\nabla_x f_\Pcal(\bx,\by^*_\Pcal(\bx))-\nabla_x f_S(\bx,\by^*_\Pcal(\bx))\|
=\widetilde{\Ocal}\left(L_0^f\sqrt{{d_x\log(D/\gamma)}/{n}}\right)~.
\end{align*}
To bound the second term, we use strong-convexity to get
\begin{align*}
L_1^f\|\by^*_\Pcal(\bx)-\by^*_S(\bx)\|
\leq \frac{L_1^f}{\sqrt{\mu_g}}\sqrt{g_\Pcal(\bx,\by^*_\Pcal(\bx))-g_\Pcal(\bx,\by^*_S(\bx))}~.
\end{align*}
Now, we note that since $\by^*_\Pcal(\bx),\by^*_S(\bx)$ are the stochastic/empirical minima of the strongly-convex stochastic optimization problem given by $g_\Pcal$, 
a result due to
\citet{feldman2019high}
bounds the generalization gap with probability 
at least $1-\gamma/9$ by
\[
g_\Pcal(\bx,\by^*_\Pcal(\bx))-g_\Pcal(\bx,\by^*_S(\bx))=\widetilde{\Ocal}\left(
\frac{(L_0^g)^2\log(1/\gamma)}{\mu_g n}
+\frac{\sqrt{\log(1/\gamma)}}{\sqrt{n}}\right)
~.
\]
Overall, by union bounding and simplifying the condition dependent constants, we get
\[
(1)
=
\widetilde{\Ocal}\left(\frac{\ell\sqrt{{d_x\log(D/\gamma)}}}{\sqrt{n}}
+\frac{\ell\kappa\sqrt{\log(1/\gamma)}}{\sqrt{n}}
+\frac{\sqrt{\ell\kappa}\log^{1/4}(1/\gamma)}{n^{1/4}}\right)~.
\]

\paragraph{Bound $(2)$.}
On one hand, by the implicit function theorem,
\[
\|\frac{d\by^*_\Pcal(\bx)}{d\bx}\|
=\|\nabla^2_{xy}g_\Pcal(\bx,\by_\Pcal^*(\bx))[\nabla^2_{yy}g_\Pcal(\bx,\by_\Pcal^*(\bx))]^{-1}\|
\leq L_1^g/\mu_g\leq \kappa~.
\]
Moreover, repeating the argument used to bound $(1)$, we get
\begin{align*}
\|\nabla_y f_\Pcal(\bx,\by^*_\Pcal(\bx))-\nabla_y f_S(\bx,\by^*_S(\bx))\|
&\leq \|\nabla_y f_\Pcal(\bx,\by^*_\Pcal(\bx))-\nabla_y f_S(\bx,\by^*_\Pcal(\bx))\|+L_1^f\|\by^*_\Pcal(\bx)-\by^*_S(\bx)\|
\\&=\widetilde{\Ocal}\left(\frac{\ell\sqrt{{d_x\log(D/\gamma)}}}{\sqrt{n}}
+\frac{\ell\kappa\sqrt{\log(1/\gamma)}}{\sqrt{n}}
+\frac{\sqrt{\ell\kappa}\log^{1/4}(1/\gamma)}{n^{1/4}}\right)~.
\end{align*}
Overall, we get
\[
(2)=\widetilde{\Ocal}\left(\frac{\ell\kappa\sqrt{{d_x\log(D/\gamma)}}}{\sqrt{n}}
+\frac{\ell\kappa^2\sqrt{\log(1/\gamma)}}{\sqrt{n}}
+\frac{\ell^{1/2}{\kappa}^{3/2}\log^{1/4}(1/\gamma)}{n^{1/4}}\right)~.
\]

\paragraph{Bound $(3)$.}
We first note that $\|\nabla_y f_S(\bx,\by^*_S(\bx))\|\leq L_0^f\leq \ell$.
Moreover, due to the implicit function theorem, the basic estimate $|ab-cd|\leq |a||b-d|+|d||a-c|$, the fact $\|A^{-1}-B^{-1}\|\leq\frac{1}{\lambda_{\min}(A)\lambda_{\min}(B)}\|A-B\|$, and second-order smoothness of $g$,
we get
\begin{align*}
\|\frac{d\by^*_\Pcal(\bx)}{d\bx}-\frac{d\by^*_S(\bx)}{d\bx}\|
&=\|\nabla^2_{xy}g_\Pcal(\bx,\by_\Pcal^*(\bx))[\nabla^2_{yy}g_\Pcal(\bx,\by_\Pcal^*(\bx))]^{-1}
-\nabla^2_{xy}g_S(\bx,\by_S^*(\bx))[\nabla^2_{yy}g_S(\bx,\by_S^*(\bx))]^{-1}
\|
\\&\leq \|\nabla^2_{xy}g_\Pcal(\bx,\by_\Pcal^*(\bx))\|\cdot \|[\nabla^2_{yy}g_\Pcal(\bx,\by_\Pcal^*(\bx))]^{-1}-[\nabla^2_{yy}g_S(\bx,\by_S^*(\bx))]^{-1}\|
\\&~~~~+\|[\nabla^2_{yy}g_S(\bx,\by_S^*(\bx))]^{-1}\|\cdot \|\nabla^2_{xy}g_\Pcal(\bx,\by_\Pcal^*(\bx))-\nabla^2_{xy}g_S(\bx,\by_S^*(\bx))\|
\\&\leq (\frac{L_1^g}{\mu_g^2}+\frac{1}{\mu_g})\|\nabla^2g_\Pcal(\bx,\by_\Pcal^*(\bx))-\nabla^2g_S(\bx,\by_S^*(\bx))\|
\\&\leq(\frac{L_1^g}{\mu_g^2}+\frac{1}{\mu_g})(\|\nabla^2g_\Pcal(\bx,\by_\Pcal^*(\bx))-\nabla^2g_S(\bx,\by_\Pcal^*(\bx))\|+L_2^g\|\by_\Pcal^*(\bx)-\by_S^*(\bx)\|)~.
\end{align*}
Applying a Hessian uniform convergence bound due to \citet[Theorem 1.b]{mei2018landscape} gives with probability at least $1-\gamma/9:$
\begin{align*}
\|\nabla^2g_\Pcal(\bx,\by_\Pcal^*(\bx))-\nabla^2g_S(\bx,\by_\Pcal^*(\bx))\|
=\widetilde{\Ocal}\left(L_1^g\sqrt{{d_x\log(D/\gamma)}/{n}}\right)~.
\end{align*}
The same argument that appeared in bounding $(1)$ and $(2)$ gives
\[
L_2^g\|\by_\Pcal^*(\bx)-\by_S^*(\bx)\|
=\widetilde{\Ocal}\left(\frac{\ell\kappa\sqrt{\log(1/\gamma)}}{\sqrt{n}}
+\frac{\sqrt{\ell\kappa}\log^{1/4}(1/\gamma)}{n^{1/4}}\right)~.
\]
Combining all the pieces and simplifying the condition dependent constants, we get overall
\[
(3)=\widetilde{\Ocal}\left(\frac{(\ell^2\kappa^2+\ell\kappa)\sqrt{d_x\log(D/\gamma)}}{\sqrt{n}}
+\frac{(\ell^2\kappa^3+\ell\kappa^2)\sqrt{\log(1/\gamma)}}{\sqrt{n}}
+\frac{(\ell^{3/2}\kappa^{5/2}+\ell^{1/2}\kappa^{3/2})\log^{1/4}(1/\gamma)}{n^{1/4}}
\right)~.
\]

\end{proof}

With Lemma~\ref{lem: uc} in hand, we turn to prove Theorem~\ref{thm: population}.

\begin{proof}[{Proof of Theorem~\ref{thm: population}}]
Using the fact that the gradient mapping is non-expansive (Lemma~\ref{lem: G is Lip}), we see that with probability at least $1-\gamma/2$ the following holds:
\begin{align*}
\norm{\Gcal_{F_\Pcal,\eta}(\bx_\out)}
&\leq \norm{\Gcal_{F_S,\eta}(\bx_\out)}+\norm{\Gcal_{F_\Pcal,\eta}(\bx_\out)-\Gcal_{F_S,\eta}(\bx_\out)}
\\
&\leq\norm{\Gcal_{F_S,\eta}(\bx_\out)}+\norm{\nabla F_{\Pcal}(\bx_\out)-\nabla F_{S}(\bx_\out)}
\\
&=\norm{\Gcal_{F_S,\eta}(\bx_\out)}+\widetilde{\Ocal}\bigg(\frac{(\ell^2\kappa^2+\ell\kappa)\sqrt{d_x\log(1/\gamma)}}{\sqrt{n}}
+\frac{(\ell^2\kappa^3+\ell\kappa^2)\sqrt{\log(1/\gamma)}}{\sqrt{n}}
\\&~~~~~~~~~~~~~~~~~~~~~~~~~~~~~~~~~~~~~~~~~+\frac{(\ell^{3/2}\kappa^{5/2}+\ell^{1/2}\kappa^{3/2})\log^{1/4}(1/\gamma)}{n^{1/4}}
\bigg)~,
\end{align*}
where the alst inequality is by Lemma~\ref{lem: uc}
with a domain bound $\|\bx_\out-\bx_0\|\leq D$ for some sufficiently large $D$ which is polynomial in all problem parameters (therefore only affecting log terms).
Theorem~\ref{thm: population} follows from applying Theorems~\ref{thm: main} and \ref{thm: main minib} to bound $\norm{\Gcal_{F_S,\eta}(\bx_\out)}$ and simplifying.

\end{proof}

\section{Discussion}

In this paper, we studied DP bilevel optimization, and proposed the first algorithms to solve this problem designed for the central DP model, while also being the first that use only gradient queries.
Our provided guarantees hold both for constrained and unconstrained settings, cover empirical and population losses alike, and account for mini-batched gradients.
As an application, we derived a private update rule for tuning a regularization hyperparameter when fitting a statistical model.

Our work leaves open several directions for future research.
First, it is natural to ask what is the extent to which our results can be improved. Notably, even in the well-studied setting of \emph{single}-level smooth-nonconvex DP optimization, there still exist gaps between known upper and lower bounds for minimizing the gradient norm (cf. \citealt{lowy2024make} and discussion therein). The best known lower bound for such problems, which trivially applies also for DP BO which is a strictly more general problem setting, is $\Omega(\sqrt{d}/\epsilon n)$, hinting that our results are perhaps not tight. Moreover, \citet{lowy2025differentially} recently presented improved gradient bounds for DP BO via second-order methods, and it is interesting to ask whether first-order methods for this setting, as studied in this work, can be improved.

To follow up on this question, we remark that a candidate strategy to improve the convergence rate in this work would be to use variance reduction, as \citet{arora2023faster} used variance reduction
to derive
faster convergence to private stationarity for single-level DP nonconvex optimization.
Applying this for DP BO as the outer loop would require, according to our analysis, to evaluate the cost of inexact gradients in variance-reduced methods, which is left for future work.

Another direction for future work
is extending the analysis in
Section~\ref{sec: regularization tuning} to derive update rules for privately tuning multiple hyperparameters simultaneously. 
It is interesting to note that the complexity of our results scales polynomially with the upper-level dimension, which corresponds to the number of hyperparameters, whereas preforming a grid search scales exponentially with the number of hyperparameters. The trade-off is, of course, convergence to local stationarity instead of global optimality.

Lastly, another open direction is 
understanding whether mini-batch algorithms can avoid the additive $1/b_{\mathrm{out}}$ factor in the \emph{unconstrained} case $\Xcal=\reals^{d_x}$.
As previously discussed, for constrained problems, even single-level nonconvex algorithms suffer from this batch dependence \citep{ghadimi2016mini}.
However, for unconstrained problems,
\citet{ghadimi2013stochastic} showed that setting a smaller stepsize, on the order of $\alpha^2/\sigma^2$, converges to a point with gradient bounded by $\alpha$ after $\Ocal(\alpha^{-4})$ steps, even for $b_{\mathrm{out}}=1$. Applying this to DP bilevel unconstrained optimization seems feasible, and requires
accounting for the larger privacy loss due to the slower convergence rate (compared to $\Ocal(\alpha^{-2})$ in our case).

\subsubsection*{Acknowledgments}
This research is supported through an Azrieli Foundation graduate fellowship. We would like to express our gratitude to Vitaly Feldman, Kunal Talwar and the Apple MLR team for the encouraging environment in which this research was initiated. Specifically, we thank Kunal Talwar for insightful discussions, which included pointing out the utility of localizing DP-SGD for optimal high probability guarantees.

\bibliographystyle{plainnat}
\bibliography{zzbib}

\newpage

\appendix

\section{Differential privacy preliminaries} \label{sec: dp prelim}

We recall here some well-known DP facts.
The basic composition property of DP states that the (possibly adaptive) composition of $(\epsilon_0,\delta_0)$-DP- and $(\epsilon_1,\delta_1)$-DP mechanisms, is $(\epsilon_0+\epsilon_1,\delta_0+\delta_1)$-DP. Advanced composition improves this scaling in typical parameter regimes of interest:

\begin{lemma}[Advanced composition, \citealp{dwork2010boosting}]
\label{lem: advanced comp}
For $\epsilon_0<1$, a $T$-fold (possibly adaptive) composition of $(\epsilon_0,\delta_0)$-DP mechanisms is $(\epsilon,\delta)$-DP for $\epsilon=\sqrt{2T\log(1/\delta_0)}\epsilon_0+2T\epsilon_0^2,~\delta=(T+1)\delta_0$.
\end{lemma}

We remark that advanced composition is typically used when
$\epsilon_0\lesssim \sqrt{\log(1/\delta_0)/T}$, thus the accumulated privacy scales as $\epsilon\asymp\sqrt{T}\epsilon_0$.

\begin{lemma}[Gaussian mechanism]
\label{lem: gaussian mech}
Given a function $h:\Xi^b\to\reals^d$, the Gaussian mechanism $\Mcal(h):\Xi^b\to\reals^d$ defined as $\Mcal(h)(S):=h(S)+\Ncal(\bzero,\sigma^2 I_d)$ is $(\epsilon,\delta)$-DP for $\epsilon,\delta\in(0,1)$,
as long as $\sigma^2\geq\frac{2\log(5/4\delta)(\Scal_h)^2}{\epsilon^2}$, where $\Scal_h:=\sup_{S\sim S'}\norm{h(S)-h( S')}$ is the $L_2$-sensitivity of $h$.
\end{lemma}

\begin{lemma}[Privacy amplification, \citealp{balle2018privacy}]
\label{lem: amplification}
Suppose $\Mcal:\Xi^b\to\Rcal$ is $(\epsilon_0,\delta_0)$-DP. Then given $n\geq b$, the mechanism $\Mcal':\Xi^n\to\Rcal,~\Mcal'(S):=\Mcal(B)$ where $B\sim\Unif(\Xi)^b$, is $(\epsilon,\delta)$-DP for 
$\epsilon=\log(1+(1-(1-1/n)^b)(e^{\epsilon_0}-1)),~\delta=\delta_0$.
\end{lemma}

Privacy amplification is typically used when
$\epsilon_0\leq 1$, under which the privacy after subsampling scales as $\epsilon\asymp\frac{b\epsilon_0}{n}$ (since $e^{\epsilon_0}-1\asymp \epsilon_0,~(1-1/n)^b\asymp\frac{b}{n}$ and $\log(1+\frac{b}{n}\epsilon_0)\asymp \frac{b}{n}\epsilon_0$).

\section{Optimal DP algorithm for strongly-convex objectives} \label{sec: DPGD}

The goal of this appendix is to provide a self contained analysis of a DP algorithm 
for strongly-convex optimization
which achieves the optimal convergence rate with a
a high probability guarantee.
Any such algorithm can be used as the inner loop in our DP bilevel algorithm.

In particular, we analyze \emph{localized} DP (S)GD.
Although it would have been more natural to apply DP-SGD, this seems (at least according to our analysis) to yield an inferior rate with respect to the required high probability guarantee.\footnote{
Even if we would have sought only expectation bounds with respect to the hyperobjective, the high probability bound with respect to the inner problem is key to being able to argue about the gradient inexactness thereafter.}
Indeed, for DP-(S)GD,
previous works (such as \citealt{bassily2014private})
typically provide bounds in expectation, and then convert them into high-probability bounds via a black-box reduction, which applies several runs and selects the best run via the private noisy-min (via Laplace mechanism).
The additional error incurred by this selection is of order $\frac{1}{n}$,
which translates to $\frac{1}{\sqrt{n}}$ in terms of distance to the optimum,
thus spoiling the fast rate of $\frac{1}{n}$ otherwise achieved in expectation
for strongly-convex objectives.
We therefore resort to localization \citep{feldman2020private}: by running projected-(S)GD over balls with shrinking radii, applying martingale concentration bounds enables us to show that the distance to optimum shrinks as $R_{m+1}\lesssim \sqrt{\frac{R_m}{n}}+\frac{1}{n}$, and thus with negligible overhead we eventually recover
the optimal fast rate $R_M\lesssim \frac{1}{n}$ with high probability.
Our analysis differs than previous localization analyses, as it does not require adapting the noise-level and step sizes throughout the rounds.

We prove the following (which easily implies also the full-batch Theorem~\ref{thm: DPGD}):

\begin{restatable}{theorem}{DPLocSGD}
\label{thm: DPLocSGD}
Suppose that $h:\reals^{d_y}\times\Xi\to\reals$ is a $\mu$-strongly-convex function of the form $h(\by)=\frac{1}{n}\sum_{i=1}^{n}h(\by,\xi_i)$ where
$h(\cdot,\xi_i)$ is $L$-Lipschitz for all $i\in[n]$.
Suppose $\arg\min h=:\by^*\in\mathbb{B}(\by_0,R_0)$,
and that $n\geq\frac{L R_0^{\frac{2}{\log(d_y)}}}{\mu\epsilon'}$.
Then given any batch size $b\in\{1,\dots,n\}$,
there is an assignment of parameters
$M=\log_2\log(\frac{\mu\epsilon'n}{L}),~\sigma^2_{\mathrm{SGD}}=\widetilde{\Ocal}\left(\frac{L^2}{\epsilon'^2}\right),~\eta_t=\frac{1}{\mu (t+1)},~T_{\mathrm{SGD}}=n^2,~R_{m}=\widetilde{\Theta}\left(\sqrt{\frac{R_{m-1} L}{\mu\epsilon' n}}
+\frac{L\sqrt{d_y}}{\mu\epsilon' n}
\right)$ such that running Algorithm~\ref{alg: DPLocSGD} satisfies $(\epsilon,\delta)$-DP, and outputs $\by_{\out}$ such that $\norm{\by_{\out}-\by^*}=\widetilde{\Ocal}\left(
\frac{L\sqrt{d_y}}{\mu n\epsilon}
\right)$ with probability at least $1-\gamma$.
\end{restatable}

\begin{proof}[Proof of Theorem~\ref{thm: DPLocSGD}]
We start by proving the privacy guarantee.
By the Lipschitz assumption, the sensitivity of $\nabla h(\cdot;B_t)$ is at most $\frac{2L}{b}$,
thus the ``unamplified'' Gaussian mechanism (Lemma \ref{lem: gaussian mech}) ensures $(\tilde{\epsilon},\tilde{\delta})$-DP with $\tilde{\epsilon}=\widetilde{\Theta}\left(\frac{L}{b\sigma}\right)=\widetilde{\Theta}\left(\frac{\epsilon'}{b}\right)$, and hence is amplified (Lemma \ref{lem: amplification}) to $(\epsilon_0,\delta_0)$-DP for $\epsilon_0=\widetilde{\Theta}\left(\frac{\epsilon'}{b}\cdot\frac{b}{n}\right)=\widetilde{\Theta}\left(\frac{\epsilon'}{n}\right)=\widetilde{\Theta}\left(\frac{\epsilon'}{\sqrt{T}}\right)$. Advanced composition (Lemma~\ref{lem: advanced comp}) therefore ensures that the overall algorithm is $(\epsilon',\delta')$-DP
(note that this uses the fact that $M=\widetilde{\Ocal}(1)$).

We turn to prove the utility of the algorithm.
We first show that for all $m:$
\begin{equation} \label{eq: induction over m}
\Pr\left[\by^*\in\mathbb{B}(\by_0^m,R_m)\right]
\geq
1-\frac{m\gamma}{M}~.
\end{equation}
We prove this by induction over $m$. The base case $m=0$ follows by the assumption $\by^*\in\mathbb{B}(\by_0,R_0)$.
Denoting $\be_t^m:=\nabla h(\by_t^m;B_t)-\nabla h(\by_t^m)$,
using the inductive hypothesis that $\by^*\in\mathbb{B}(\by_0^m,R_m)$ with probability at least $1-\frac{m\gamma}{M}$,
under this probably event we get

\begin{align*}
\norm{\by_{t+1}^m-\by^*}^2
&=\norm{\mathrm{Proj}_{\mathbb{B}(\by_0^m,R_
m)}\left[\by_t^m-\eta_t(\nabla h(\by_t^m;B_t^m)+\nu_t^m)\right]-\by^*}^2
\\&\leq\norm{\by_t^m-\eta_t(\nabla h(\by_t^m;B_t^m)+\nu_t^m)-\by^*}^2
\\&=\norm{\by_t^m-\by^*}^2-2\eta_t\inner{\by_t^m-\by^*,\nabla h(\by_t^m;B_t^m)+\nu_t^m}+\eta_t^2\norm{\nabla h(\by_t^m;B_t^m)+\nu_t^m}^2
\\&\leq \norm{\by_t^m-\by^*}^2-2\eta_t\inner{\by_t^m-\by^*,\nabla h(\by_t)+\be^m_t+\nu_t^m}+2\eta_t^2\left(\norm{\nu^m_t}^2+\norm{\nabla h(\by_t)}^2\right)
\\
&=\norm{\by_t^m-\by^*}^2-2\eta_t\inner{\by_t^m-\by^*,\nabla h(\by_t^m)}
\\
&~~~~~-2\eta_t\inner{\by_t^m-\by^*,\be^m_t+\nu_t^m}+2\eta_t^2\left(\norm{\nu^m_t}^2+\norm{\nabla h(\by_t)}^2\right)~.
\end{align*}
Rearranging, and using the strong convexity and Lipschitz assumptions,
we see that
\begin{align*}
h(\by_t^m)-h(\by^*)
&\leq
\inner{\by_t^m-\by^*,\nabla h(\by_t^m)}-\frac{\mu}{2}\norm{\by_t^m-\by^*}^2
\\
&\leq
\left(\frac{1}{2\eta_t}-\frac{\mu}{2}\right)\norm{\by_t^m-\by^*}^2
-\frac{1}{2\eta_t}\norm{\by_{t+1}^m-\by^*}^2
\\
&~~~~~~-\inner{\by_t^m-\by^*,\be^m_t+\nu_t^m}
+\eta_t\left(\norm{\nu_t^m}^2+L^2\right)~.
\end{align*}
Averaging over $t$ and using $\eta_t=\frac{1}{\mu (t+1)}$, which satisfies $\left(\frac{1}{\eta_t}-\frac{1}{\eta_{t-1}}-\mu\right)\leq 0$ and $\frac{1}{T}\sum_{t=0}^{T}\eta_t\lesssim\frac{\log T}{\mu T}$,
by Jensen's inequality, overall we get with probability at least $1-\frac{m\gamma}{M}:$

\begin{align}
h(\by_0^{m+1})-h(\by^*)
&=h\left(\frac{1}{T}\sum_{t=0}^{T-1}\by_t^m\right)-h(\by^*) \nonumber
\\&\lesssim \underset{(I)}{\underbrace{\left|\frac{1}{T}\sum_{t=0}^{T-1}\inner{\by_t^m-\by^*,\be^m_t+\nu_t^m}\right|}}
+\frac{L^2 \log T}{\mu T}+
\frac{\log T}{\mu T}\underset{(II)}{\underbrace{\sum_{t=0}^{T-1}\norm{\nu_t^m}^2}}
~.\label{eq: h SGD bound}
\end{align}

We now apply concentration inequalities to bound $(I)$ and $(II)$ with high probability, for which we will use basic properties of sub-Gaussian distributions (cf. \citealt[\S3.4]{vershynin2018high}).
To bound $(I)$, note that
for all $t:~\E\be^m_t=\E\nu^m_t=\bzero$ and therefore $\E\inner{\by_t^m-\by^*,\be^m_t+\nu_t^m}=0$. Moreover, $\be^m_t=\frac{1}{b}\sum_{\xi\in B_t^m}(\nabla h(\bx_t^m;\xi)-\nabla h(\by_t^m))$ is the average of $b$ independent vectors with norm bounded by at most $2L$, while $\nu_t^m\sim\Ncal(\bzero,\sigma^2_{\mathrm{SGD}}I_{d_y})=\Ncal(\bzero,\widetilde{\Ocal}(\frac{L^2}{\epsilon'^2})\cdot I_{d_y})$, and also $\norm{\by_t^m-\by^*}\leq R_m$ by the inductive hypothesis. By combining all of these observations, we see that $\inner{\by_t^m-\by^*,\be^m_t+\nu_t^m}$ is a $\Ocal(R_m\cdot (\frac{L}{b}+\frac{L}{\epsilon'}))=\Ocal(\frac{R_m L}{\epsilon'})$-sub-Gaussian random variable.
By Azuma's inequality for sub-Gaussians \citep{shamir2011variant}, we get that with probability at least $1-\frac{\gamma}{2M}:$
\begin{equation} \label{eq: bound SGD (I)}
    (I)=\widetilde{\Ocal}\left(\frac{\frac{R_m L}{\epsilon'}\log(\gamma/M)}{\sqrt{T}}\right)
=\widetilde{\Ocal}\left(\frac{R_m L}{\epsilon' n}\right)~.
\end{equation}
To bound $(II)$, by concentration of the Gaussian norm (cf. \citealt[Theorem 3.1.1]{vershynin2018high}) and the union bound we can get that with probability at least $1-\frac{\gamma}{2M}:$
\begin{equation} \label{eq: bound SGD (II)}
(II)
=\widetilde{\Ocal}\left(d_y\sigma_{\mathrm{SGD}}^2\right)
=\widetilde{\Ocal}\left(\frac{d_yL^2}{\epsilon'^2}\right)~.
\end{equation}
Plugging Eqs. (\ref{eq: bound SGD (I)}) and (\ref{eq: bound SGD (II)}) into \eqref{eq: h SGD bound}, we overall get that with probability at least $1-\frac{m\gamma}{M}-2\cdot\frac{\gamma}{2M}=1-\frac{(m+1)\gamma}{M}:$
\[
h(\by_0^{m+1})-h(\by^*)
=\widetilde{\Ocal}\left(\frac{R_m L}{\epsilon' n}+\frac{d_yL^2}{\mu \epsilon'^2 n^2}\right)~.
\]
Applying the $\mu$-strong-convexity of $h$, and sub-additivity of the square root, we get that
\[
\norm{\by_0^{m+1}-\by^*}
\leq \sqrt{\frac{2(h(\by_0^{m+1})-h(\by^*))}{\mu}}
=\widetilde{\Ocal}\left(\sqrt{\frac{R_m L}{\mu\epsilon' n}}
+\frac{L\sqrt{d_y}}{\mu\epsilon' n}
\right)
\leq R_{m+1}
~.
\]
We have therefore proven \eqref{eq: induction over m}.
In particular for $m=M$ we get that with probability at least $1-\gamma:$
\begin{equation} \label{eq: R_M distance bound}
\norm{\by_\out-\by^*}\leq R_{M}~,
\end{equation}
hence it remains to bound $R_M$. We will prove, once again by induction over $m$, that
\begin{equation} \label{eq: R_m induction}
R_m=\widetilde{\Ocal}\left(R_0^{\frac{1}{2^m}}\left(\frac{L}{\mu\epsilon'n}\right)^{1-\frac{1}{2^m}}+\frac{L}{\mu\epsilon'n}\sum_{i=1}^{m}d_y^{\frac{1}{2^i}}\right)~.
\end{equation}
The base $m=0$ simply follows since the left hand side in \eqref{eq: R_m induction}
reduces to $R_0$.
Denoting $A:=\frac{L}{\mu\epsilon' n}$, by our assignment of $R_{m+1}$, the induction hypothesis and sub-additivitiy of the square root we get:
\begin{align*}
R_{m+1}&=\widetilde{\Ocal}\left(\sqrt{R_m A}
+A\sqrt{d_y}
\right)
\\
&=\widetilde{\Ocal}\left(A^{1/2}\left(
R_0^{\frac{1}{2^{m+1}}}A^{\frac{1}{2}-\frac{1}{2^{m+1}}}+A^{1/2}\sum_{i=1}^{m}d_y^{\frac{1}{2^{i+1}}}
\right)
+A d_y^{1/2}\right)
\\
&=\widetilde{\Ocal}\left(R_0^{\frac{1}{2^{m+1}}}A^{1-\frac{1}{2^{m+1}}}+A\sum_{i=1}^{m+1}d_y^{\frac{1}{2^{i}}}
\right)~,
\end{align*}
therefore proving \eqref{eq: R_m induction}. In particular, for $m=M=\log_2\log(\frac{\mu\epsilon'n}{L})$,
which satisfies $\frac{1}{2^M}=\frac{1}{\log(\frac{\mu\epsilon'n}{L})}
=\frac{1}{\log(1/A)}$
we get
\begin{align*}
R_M&=\widetilde{\Ocal}\left(R_0^{\frac{1}{\log(1/A)}}
A^{1+\frac{1}{\log(A)}}
+MAd_y^{1/2}
\right)
\\
&=\widetilde{\Ocal}\left(R_0^{\frac{1}{\log(1/A)}}
A+Ad^{1/2}
\right)
\\
&=\widetilde{\Ocal}\left(R_0^{\frac{1}{\log(\mu\epsilon'n/L)}}
\frac{L}{\mu\epsilon' n}+\frac{Ld_y^{1/2}}{\mu\epsilon' n}
\right)
\\
&=\widetilde{\Ocal}\left(\frac{L \sqrt{d_y}}{\mu\epsilon' n}\right)~,
\end{align*}
where the last follows from our assumption on $n$.
This completes the proof by \eqref{eq: R_M distance bound}.
\end{proof}

\section{Auxiliary lemma} \label{sec: aux lemmas}

We will recall a useful fact, which asserts that the mapping $\Gcal_{\bv,\eta}(\bx)$ is non-expansive with respect to $\bv:$

\begin{restatable}{lemma}{GLip}
\label{lem: G is Lip}
For any $\bx,\bv,\bw\in\reals^d,~\eta>0:~\norm{\Gcal_{\bv,\eta}(\bx)-\Gcal_{\bw,\eta}(\bx)}\leq\norm{\bv-\bw}$.
\end{restatable}

In our analysis we argue that in the bilevel setting,
gradient estimates must be \emph{inexact} 
to avoid privacy leaking between the two levels, and Lemma~\ref{lem: G is Lip} allows us to control the error due to this inexactness.
Lemma~\ref{lem: G is Lip} is known (cf. \citealt{ghadimi2016mini}), and we reprove
it
here
for completeness.

\begin{proof}[Proof of Lemma~\ref{lem: G is Lip}]
The proof is due to \citet{ghadimi2016mini}.
By definition,
\begin{align*}
\Pcal_{\bv,\eta}(\bx)&=\arg\min_{\bu\in\Xcal}
\left\{\sinner{\bv,\bu}+\frac{1}{2\eta}\norm{\bx-\bu}^2\right\}~,
\\
\Pcal_{\bw,\eta}(\bx)&=\arg\min_{\bu\in\Xcal}
\left\{\sinner{\bw,\bu}+\frac{1}{2\eta}\norm{\bx-\bu}^2\right\}~,
\end{align*}
hence by first order optimality criteria, for any $\bu\in\Xcal:$
\begin{align*}
\inner{\bv+\frac{1}{\eta}(\Pcal_{\bv,\eta}(\bx)-\bx),\bu-\Pcal_{\bv,\eta}(\bx)}&\geq 0~,
\\
\inner{\bw+\frac{1}{\eta}(\Pcal_{\bw,\eta}(\bx)-\bx),\bu-\Pcal_{\bw,\eta}(\bx)}&\geq 0~.
\end{align*}
Plugging $\Pcal_{\bw,\eta}(\bx)$ as $\bu$ into the first inequality above, and $\Pcal_{\bv,\eta}(\bx)$ into the second,
shows that
\begin{align*}
0&\leq\inner{\bv+\frac{1}{\eta}(\Pcal_{\bv,\eta}(\bx)-\bx),\Pcal_{\bw,\eta}(\bx)-\Pcal_{\bv,\eta}(\bx)}~,
\\
0&\leq
\inner{\bw+\frac{1}{\eta}(\Pcal_{\bw,\eta}(\bx)-\bx),\Pcal_{\bv,\eta}(\bx)-\Pcal_{\bw,\eta}(\bx)}
=\inner{-\bw+\frac{1}{\eta}(\bx-\Pcal_{\bw,\eta}(\bx)),\Pcal_{\bw,\eta}(\bx)-\Pcal_{\bv,\eta}(\bx)}
~.
\end{align*}
Summing the two inequalities yields
\begin{align*}
0&\leq
\inner{{\bv-\bw}+\frac{1}{\eta}(\Pcal_{\bv,\eta}(\bx)-\Pcal_{\bw,\eta}(\bx)),\Pcal_{\bw,\eta}(\bx)-\Pcal_{\bv,\eta}(\bx)}
\\
&=\inner{\bv-\bw,\Pcal_{\bw,\eta}(\bx)-\Pcal_{\bv,\eta}(\bx)}-\frac{1}{\eta}\norm{\Pcal_{\bw,\eta}(\bx)-\Pcal_{\bv,\eta}(\bx)}^2
\\
&\leq\norm{\Pcal_{\bw,\eta}(\bx)-\Pcal_{\bv,\eta}(\bx)}\left(\norm{\bv-\bw}-\frac{1}{\eta}\norm{\Pcal_{\bw,\eta}(\bx)-\Pcal_{\bv,\eta}(\bx)}\right)
~.
\end{align*}
Hence, 
\begin{align*}
\norm{\bv-\bw}
&\geq
\frac{1}{\eta}\norm{\Pcal_{\bv,\eta}(\bx)-\Pcal_{\bw,\eta}(\bx)}
\\&=\norm{\frac{1}{\eta}\left(\Pcal_{\bv,\eta}(\bx)-\bx\right)-\frac{1}{\eta}\left(\Pcal_{\bw,\eta}(\bx)-\bx\right)}
\\&=\norm{\Gcal_{\bv,\eta}(\bx)-\Gcal_{\bw,\eta}(\bx)}~.
\end{align*}
\end{proof}

\end{document}